\documentclass[twoside]{article}

\usepackage[accepted]{aistats2026}
\usepackage[utf8]{inputenc} % allow utf-8 input
\usepackage[T1]{fontenc}    % use 8-bit T1 fonts
\usepackage{lmodern}        % scalable Latin Modern fonts for pdfTeX/microtype
\usepackage[colorlinks=true,linkcolor=black,citecolor=black,urlcolor=blue]{hyperref}       % hyperlinks
\usepackage{url}            % simple URL typesetting
\usepackage{caption}
\usepackage{subcaption}
\usepackage{booktabs}       % professional-quality tables
\usepackage{float}
\usepackage{amsfonts}       % blackboard math symbols
\usepackage{nicefrac}       % compact symbols for 1/2, etc.
\usepackage{microtype}      % microtypography
\usepackage{graphicx}
\usepackage{xcolor}
\newcommand{\mintinline}[2]{\texttt{\detokenize{#2}}}
\graphicspath{{figures/}}
\usepackage{amssymb}
\usepackage{mathtools}
\usepackage{amsthm}

\usepackage[capitalize,noabbrev,nameinlink]{cleveref}
\let\autoref\Cref

\theoremstyle{definition}
\newtheorem{lemma}{Lemma}

\newtheorem{theorem}{Theorem}
\newtheorem{definition}{Definition}
\usepackage{mathtools,amsfonts, bm}

\newcommand{\act}{\mathbin{\bm{\cdot}}}

\newcommand{\vs}{\mathbf{s}}
\newcommand{\vd}{\mathbf{d}}

\newcommand{\reals}{\mathbb{R}}

\newcommand{\qs}{\mathbf{qs}}
\newcommand{\ks}{\mathbf{ks}}
\newcommand{\ve}{\mathbf{e}}

\newcommand{\vf}{\mathbf{f}}

\newcommand{\eqdist}{\stackrel{d}{=}}
\newcommand{\naturals}{\mathbb{N}}
\newcommand{\given}{\mid}

\renewcommand{\P}{\mathbb{P}}
\newcommand{\dmu}{\,\mathrm{d}\mu}
\newcommand{\calF}{\mathcal{F}}
\newcommand{\calB}{\mathcal{B}}
\newcommand{\vx}{\mathbf{x}}
\newcommand{\vt}{\mathbf{t}}

\usepackage[authoryear,round]{natbib}

\begin{document}
\runningtitle{Scalable Spatiotemporal Inference with BSA-TNP}
\runningauthor{Jenson et al.}



\twocolumn[

\aistatstitle{Scalable Spatiotemporal Inference with Biased Scan Attention Transformer Neural Processes}

\aistatsauthor{ Daniel Jenson\textsuperscript{1} \And Jhonathan Navott\textsuperscript{2,1} \And  Piotr Grynfelder\textsuperscript{1} \And Mengyan Zhang\textsuperscript{3,1} \AND Makkunda Sharma\textsuperscript{1} \And Elizaveta Semenova\textsuperscript{2} \And Seth Flaxman\textsuperscript{1} }
\vspace{0.5em}
\aistatsaddress{\textsuperscript{1}University of Oxford\quad\textsuperscript{2}Imperial College London\quad\textsuperscript{3}University of Bristol
}
]

\begin{abstract}
Neural Processes (NPs) are a rapidly evolving class of models designed to directly model the posterior predictive distribution of stochastic processes. While early architectures were developed primarily as a scalable alternative to Gaussian Processes (GPs), modern NPs tackle far more complex and data-hungry applications spanning geology, epidemiology, climate, and robotics. These applications have placed increasing pressure on the scalability of these models, with many architectures compromising accuracy for scalability. In this paper, we demonstrate that this trade-off is often unnecessary, particularly when modeling fully or partially translation-invariant processes. We propose a versatile new architecture, the Biased Scan Attention Transformer Neural Process (BSA-TNP), which introduces Kernel Regression Blocks (KRBlocks), group-invariant attention biases, and memory-efficient Biased Scan Attention (BSA). BSA-TNP is able to: (1) match or exceed the accuracy of the best models while often training in a fraction of the time, (2) exhibit translation invariance, enabling learning at multiple resolutions simultaneously, (3) transparently model processes that evolve in both space and time, (4) support high-dimensional fixed effects, and (5) scale gracefully, running inference on over 1M test points and 100K context points in under a minute on a single 24GB GPU. Code is provided as part of the \href{https://github.com/MLGlobalHealth/dl4bi}{\texttt{dl4bi}} package.
\end{abstract}

\section{INTRODUCTION}\label{sec:introduction}
While early Neural Processes (NPs) \citep{np, cnp} focused primarily on the posterior predictive distributions of 1D Gaussian Processes (GPs) and simple data distributions like MNIST, many modern NPs tackle far more expansive and complex distributions spanning geology, epidemiology, climate, and robotics \citep{geoai,stgnp,convcnp_env,convcnp_env_2,anp_robotics}. These tasks often require combining on-grid and off-grid data. Indeed, one of the most recent foundation models for climate, Aardvark, is an NP that synthesizes temperature, pressure, wind, humidity, and precipitation from on-grid and off-grid data to generate forecasts that outperform traditional numerical weather prediction systems at one thousandth the computational cost \citep{aardvark}.

As NPs grow in scope, there are two competing pressures: scale and accuracy. Scale is particularly salient for applications involving high-resolution sensor data, e.g.~satellite imagery or LiDAR point clouds. On the other hand, many of these predictions need to be locally accurate to be actionable, e.g.~city-level disaster preparedness. NPs should be simple, extensible, and computationally tractable to enable widespread use across domains with limited resources. Accordingly, we propose BSA-TNP, a model that effectively balances these desiderata. Our contributions include:

\begin{itemize}
\item The Kernel Regression Block (KRBlock): a simple, stackable transformer block that is parameter-efficient, sharing weights for query and key updates, and easily extensible, supporting an arbitrary number of attention biases derived from subsets of input features, e.g.~space, time, and fixed effects (observed covariates), such as elevation, weather, or categorical metadata.
\item Group-invariant (\textit{G}-invariant) attention biases: we present a general form of attention bias based on group-invariant actions that operate on subsets of the model input and enable specifying constraints similar to Bayesian priors. For spatiotemporal applications, we focus on translation invariance as a special case, which accelerates convergence and improves performance and generalization across locations and resolutions.
\item Biased Scan Attention (BSA), a novel memory-efficient attention mechanism that accommodates custom, high-performance bias functions through compiled JAX functions.
\end{itemize}

Together, these allow BSA-TNP to: (1) match or exceed the accuracy of the best NP models while often training in a fraction of the time, (2) exhibit translation invariance, enabling learning at multiple resolutions simultaneously, (3) transparently model processes that evolve in both space and time, (4) support high-dimensional fixed effects, and (5) scale gracefully, running inference on over 1M test points and 100K context points in under a minute on a single 24GB GPU.

\section{BACKGROUND}\label{sec:background}
\subsection{Neural Processes}
At a high level, a stochastic process is defined as a collection of random variables ${f(s): s \in \mathcal{S}}$ indexed by $\mathcal{S}$. For a finite subset of indices $\mathbf{s}$ (``locations''), we denote the corresponding function values by $\mathbf{f}$. In Neural Processes (NPs), points are further divided into context and target sets, $(\mathbf{s}_c, \mathbf{f}_c)$ and $(\mathbf{s}_t, \mathbf{f}_t)$, respectively. Conditional Neural Processes (CNPs) are a subclass of NPs that encode the context set deterministically into a fixed representation, $\mathbf{r}_c = f_\text{enc}(\mathbf{s}_c,\mathbf{f}_c)$. This representation is then decoded at the target locations to generate the parameters of the output distribution, $\theta_t = f_\text{dec}(\mathbf{r}_c, \mathbf{s}_t)$. CNPs factorize conditionally on $\mathbf{r}_c$, i.e. $p(\mathbf{f}_t \mid \mathbf{s}_c,\mathbf{f}_c,\mathbf{s}_t) = \prod_i p(\mathbf{f}_t^{(i)} \mid \mathbf{r}_c,\mathbf{s}_t^{(i)})$ where $(\mathbf{f}_t^{(i)}, \mathbf{s}_t^{(i)})$ are individual test points, and are trained by maximizing the log-likelihood of the data under this predicted distribution.

\subsection{Transformer Neural Processes}
Transformer Neural Processes (TNPs) \citep{tnp} are CNPs that have recently dominated the NP landscape in terms of accuracy and uncertainty calibration. Most transformers consist of an embedding layer, several transformer blocks, and a prediction head. The standard transformer block consists of multi-headed attention followed by a feedforward network, interspersed with residual connections. The attention mechanism \citep{attention} projects its input into three distinct matrices corresponding to queries ($\mathbf{Q}$), keys ($\mathbf{K}$), and values ($\mathbf{V}$). The queries are matched against keys using a kernel, $\mathcal{K}$, and the resulting scores are used as weights for combining the associated values. The most common attention kernel is the dot-product softmax kernel, i.e. $\mathcal{K}(\mathbf{Q},\mathbf{K})\mathbf{V}\coloneqq\text{softmax}(\mathbf{QK}^\intercal/\sqrt{d_k})\mathbf{V}$ where $d_k$ is the embedding dimension of keys.

Unlike other CNPs which typically follow an encode-decode framework, test points are often passed to TNPs alongside context points and allowed to interact with internal representations of context points at multiple levels within the transformer (see \autoref{fig:architecture}). The original TNP, i.e.~TNP-Diagonal (TNP-D) \citep{tnp}, consisted solely of encoder blocks with a special mask, while later TNPs used alternating self and cross-attention blocks as detailed in \autoref{subsec:scaling-attention}.

\subsection{Scaling Attention}\label{subsec:scaling-attention}
Standard attention mechanisms have a space and time complexity of $\mathcal{O}(n^2)$, which presents a challenge as the number of tokens or points, $n$, increases. There are five broad categories of research that attempt to address this: (1) sparsity \citep{sparse,longformer,bigbird,reformer,sinkhorn}, (2) inducing points \citep{perceiver,settransformer}, (3) low-rank approximations \citep{linformer,performer,efficient}, (4) reuse or sharing \citep{lazyformer,reusetransformer}, and (5) optimized memory and/or hardware code \citep{memory_efficient_attention,flash1,flash2,flash3,flex}.

Most contemporary NPs, such as Latent Bottlenecked Attentive Neural Processes \citep{lbanp}, Memory Efficient Neural Processes via Constant Memory Attention Blocks \citep{cmab}, Pseudo Token Translation Equivariant Transformer Neural Processes \citep{pttnp}, and Gridded Transformer Neural Processes \citep{gridtnp}, focus on (2) inducing points and use some form of Perceiver or Set Transformer attention \citep{perceiver,settransformer}, which introduce a number of latent inducing points, $k$, often much smaller than either the number of context points, $n_c$, or test points, $n_t$. While inducing points enable more scalable training, they introduce additional complexities, such as choosing the number of latents as well as their locations, which in turn determine the accuracy of model output. In contrast, BSA-TNP focuses on (5), memory and hardware-efficient attention (see \autoref{subsec:bsa}), which is more accurate and often faster in practice.
\subsection{Attention Bias}
Attention bias, $\mathbf{B}$, injects domain knowledge into the attention kernel and takes the following form:
\begin{equation}\label{biased_attention}
\mathcal{K}(\mathbf{Q},\mathbf{K})\mathbf{V}\coloneqq\text{softmax}(\mathbf{QK}^\intercal/\sqrt{d_k}+\mathbf{B})\mathbf{V}
\end{equation}
Graph Neural Networks often leverage attention bias to encode graph topology at various scales, e.g.~the Graphormer \citep{graphormer} encodes spatial and structural biases using the shortest path and node centrality statistics. Similarly, in large language modeling, \cite{alibi} introduced Attention with Linear Biases (ALiBi), which replaces positional embeddings with a bias that is proportional to the distance between tokens, i.e. $b(i,j)=-m\cdot|i-j|$ where $m$ is a fixed or learnable scalar and $i$ and $j$ are token positions. ALiBi matches the performance of sinusoidal embeddings while allowing the model to extrapolate far beyond its training range. The matrix $\mathbf{B}$ can be static or dynamic with learnable parameters, as is the case in BSA-TNP (see \autoref{subsec:g-invariant-bias}).

\section{GROUP INVARIANCE}\label{sec:G-invariance}

Learning meaningful functions from limited data is inherently difficult without strong inductive biases. For Neural Processes (NPs), the lack of structural assumptions forces the model to search over a vast, unconstrained hypothesis space, often resulting in poor generalization and extrapolation. This issue becomes especially severe when the target function possesses symmetries that the model fails to exploit, for example in the case of stationary distributions.

\begin{definition}\label{def:strictly-stationary}
A stochastic process $F(d)$, $d \in \mathcal{D} = \reals^m$ is said to be \textbf{strictly stationary} if all its finite-dimensional marginal distributions are translation invariant. That is, the equality in distribution 
\[
\left(F(d_1), \ldots, F(d_k)\right) \eqdist \left(F(d_1 + \tau), \ldots, F(d_k + \tau)\right)
\]
holds for all $k \in \naturals$, $d_1, \ldots, d_k \in \mathcal{D}$, and $\tau \in \mathcal{D}$ and $\eqdist$ signifies equality in distribution.
\end{definition}

This notion of stationarity can be extended to a broader class of symmetries by considering a group $G$ acting on the input space $\mathcal{D}$. The space $\mathcal{D}$ may include not only spatial coordinates $s\in \mathcal{S}$, but also temporal indices $t\in\mathcal{T}$ and fixed effects $x\in \mathcal{X}$, where fixed effects are observed covariates attached to each point and conditioned on by the model. For example, one may take $\mathcal{D}=\mathcal{X}\times\mathcal{S}\times\mathcal{T}$. Such a generalization allows us to capture invariances beyond translations, such as rotations, time shifts, scaling transformations, or permutations of categorical attributes, depending on the structure of $G$ and its action on $\mathcal{D}$. Let $\act : G \times \mathcal{D}\rightarrow \mathcal{D}$ be a group action, written as $g \act d$ for $g
\in G, d \in \mathcal{D}$, extending the notion to tuples over $\mathcal{D}$ by applying the action
element-wise.
\begin{definition}
A function $f$ over $\mathcal{D}$ is  \textbf{$G$-invariant} if $f (g \act \vd) = f(\vd)$
for all $\vd \in \mathcal{D}^k$ and $g \in G$.
\end{definition}
\begin{definition}\label{def:g-stationary}
A stochastic process $F$ is \textbf{$G$-stationary} if all its finite-dimensional
marginal distributions are $G$-invariant, that is $F(\vd) \eqdist F(g \act \vd)$
holds for all $\vd \in \mathcal{D}^k$ and $g \in G$. 
\end{definition}

Note that \autoref{def:strictly-stationary} is special case of \autoref{def:g-stationary}. To understand how $G$-invariance can be incorporated into an NP, we examine the structure of the posterior predictive map $\pi \colon (\vd_c, \vf_c, \vd_t) \mapsto (F(\vd_t) \given F(\vd_c) = \vf_c)$ where $\mathbf{d}_c$ represents the index set of context points, $\mathbf{d}_t$ the index set of test points, and $\mathbf{f}_c$ the function values of the context points. This map captures the conditional distribution that an NP aims to approximate. Note that this definition of a posterior predictive map differs slightly from previous works in order to be consistent with the functional definition of TNPs; this has implications for the use of the terms invariance and equivariance, which we detail in \autoref{appendix:invariance-vs-equivariance}.

The following result, which generalizes the theorem from \cite{pttnp} beyond translations to arbitrary group actions, formalizes the connection between $G$-stationarity of the underlying process and the symmetries of the prediction map:

\begin{theorem}\label{theorem:iff}
A stochastic process $F$ is $G$-stationary if and only if the posterior predictive map $\pi$ is $G$-invariant with respect to the context and target inputs, i.e.,
$
\pi(g \act \vd_c, \vf_c, g \act \vd_t) = \pi(\vd_c, \vf_c, \vd_t)
$,
for all $g \in G$ and all $\vd_c, \vf_c, \vd_t$.
\end{theorem}

This result provides a direct prescription for incorporating invariances into NPs: \textit{if the model is designed such that its predictive map is $G$-invariant in the inputs, then it correctly reflects the $G$-stationarity of the process it aims to model}. It generalizes the translation result of \cite{pttnp} to arbitrary group actions and is closely related to Proposition~1 of \cite{aenp}; our formulation is stated for the TNP posterior predictive map $\pi(\vd_c,\vf_c,\vd_t)$, so the relevant symmetry is joint $G$-invariance in the context and target inputs rather than $G$-equivariance of an encoder-decoder parameter map. This enables us to model multiple distinct $G$-invariances simultaneously within the same model. The full proof can be found in \autoref{appendix:G-inv-iff-proof}. Since many spatiotemporal tasks exhibit some degree of translation invariance, we focus primarily on that group action in our main benchmarks, but we also include a rotational-invariance example on spherical GPs in \autoref{subsec:gp-spherical}; full experimental details are deferred to \autoref{appendix:exp}.

\section{BIASED SCAN ATTENTION TRANSFORMER NEURAL PROCESS (BSA-TNP)}\label{sec:bsa-tnp}
\begin{figure*}[t]  % [t], [b], or [h] for placement
  \centering
  \includegraphics[width=0.80\linewidth]{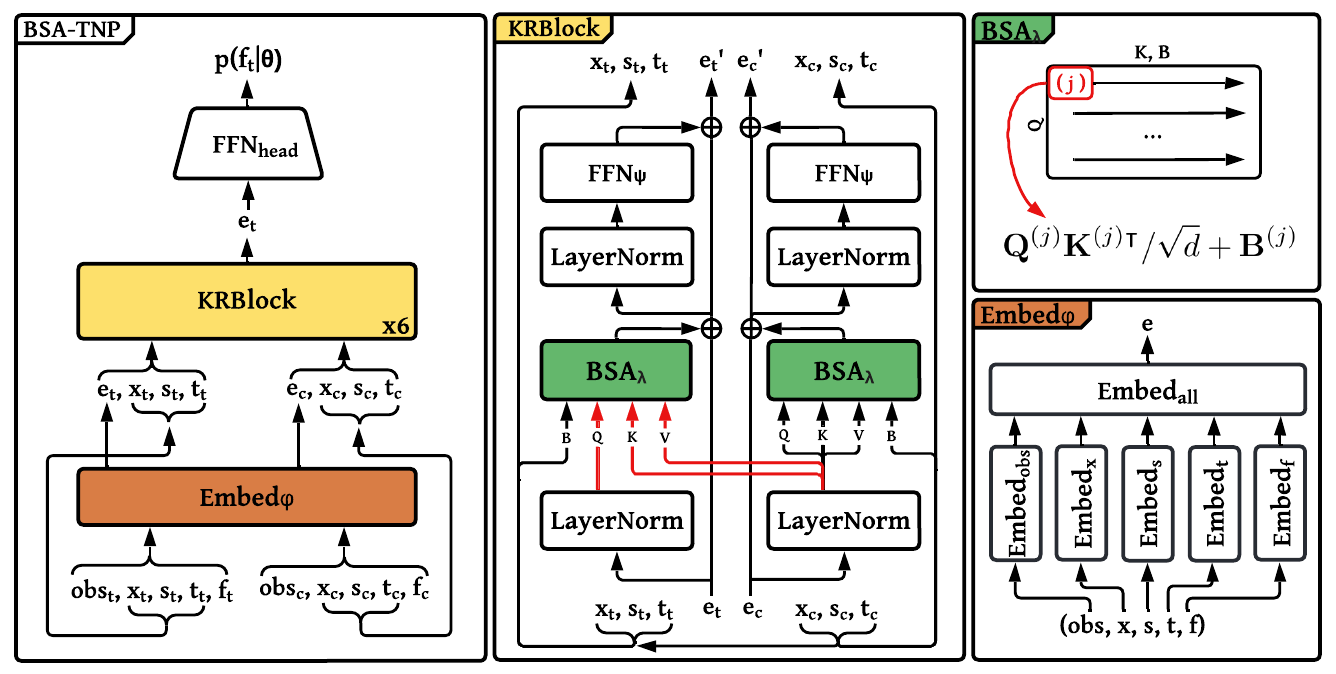}
  \caption{BSA-TNP Overview. The leftmost panel contains the high level architecture (BSA-TNP), the center panel the Kernel Regression Block (KRBlock), and the rightmost panels contain illustrations of Biased Scan Attention (BSA) and the embedding network. Subscripts in block names identify collections of parameters.}
  \label{fig:architecture}
\end{figure*}
The Biased Scan Attention Transformer Neural Process (BSA-TNP) architecture consists of an embedding layer, a stack of KRBlocks, and a prediction head (leftmost panel in \autoref{fig:architecture}). While prior NP literature primarily focused on spatial processes involving only $\mathbf{s}$ as the index set, we expand this to $\mathbf{d}=(\mathbf{x}, \mathbf{s}, \mathbf{t})$ where $\mathbf{x}$ consists of fixed effects, i.e. observed covariates associated with each spatiotemporal location, $\mathbf{s}$ consists of spatial locations, and $\mathbf{t}$ consists of temporal indices. Furthermore, $\mathbf{d}_c$ consists of the index set or features associated with context points and $\mathbf{d}_t$ consists of the index set or features associated with test points. Thus, the posterior predictive map for BSA-TNP can be written as: $\pi: (\mathbf{d}_c, \mathbf{f}_c, \mathbf{d}_t)\mapsto (F(\mathbf{d}_t)\mid F(\mathbf{d}_c)=\mathbf{f}_c)$ or $\pi(\mathbf{x}_c, \mathbf{s}_c, \mathbf{t}_c, \mathbf{f}_c, \mathbf{x}_t, \mathbf{s}_t, \mathbf{t}_t)\mapsto (F(\mathbf{x}_t,\mathbf{s}_t, \mathbf{t}_t)\mid F(\mathbf{x}_c,\mathbf{s}_c,\mathbf{t}_c)=\mathbf{f}_c)$.
% where again $\mathbf{f}_c$ represent the function values of context points and $\mathbf{f}_t$ represent the function values of test points.

The inputs to BSA-TNP consist of a boolean indicator \textbf{obs}, which enables the model to differentiate context (observed) from test (unobserved) points, the index set $\mathbf{d}=(\mathbf{x},\mathbf{s},\mathbf{t})$, and the function values $\mathbf{f}$ for both context and test points (see panel 1 in \autoref{fig:architecture}). The function values for test points, $\mathbf{f}_t$, are initialized to zero, similar to previous work \citep{tnp,pttnp}. The model outputs the parameters, $\boldsymbol{\theta}$, of the distribution at test points, $p(\mathbf{f}_t\mid\boldsymbol{\theta})$, e.g. a Gaussian for continuous outputs.

\subsection{Embedding}
The embedding layer consists of two stages. First, each element of the 5-tuple is embedded separately, e.g.~$\textbf{Embed}_\mathbf{x}(\mathbf{x})=\mathbf{e}_x$. Sometimes the embedding function is the null function, e.g.~$\textbf{Embed}_\mathbf{s}(\mathbf{s})=\mathbf{e}_s=\emptyset$, which is required in order to achieve the desired $G$-invariances (see \autoref{subsec:g-invariant-bias}). After applying the individual embedding functions, which are most often an MLP or the null function, the representations are fused with a joint embedding network, $\textbf{Embed}_\textbf{all}(\mathbf{e}_\textbf{obs}, \mathbf{e}_x, \mathbf{e}_s, \mathbf{e}_t, \mathbf{e}_f)\mapsto \mathbf{e}$. This joint embedding, $\mathbf{e}$, is passed along with the unaltered index set, $\mathbf{d}=(\mathbf{x},\mathbf{s},\mathbf{t})$, to the KRBlocks.
% The unaltered index set is later used to calculate attention bias terms (see \autoref{subsec:g-invariant-bias}).

\subsection{Kernel Regression Block (KRBlock)}\label{subsec:krblock}
A Kernel Regression Block (KRBlock) is a generic transformer block inspired by Nadaraya-Watson kernel regression \citep{nw_regression}. It consists of two computational pathways -- one for the context points and another for test points (center panel in \autoref{fig:architecture}). Both of these pathways share subnetworks and parameters, but the attention calculation differs for each. In the context point pathway, context points only attend to other context points. In the test point pathway, test points only attend to context points. Stacking KRBlocks allows the model to perform iterative kernel regression on increasingly complex internal representations of points. The inputs are joint embeddings and unaltered index sets for context and test points. The embeddings are updated by the block and the index set is used to calculate attention bias. The output consists of updated embeddings and the unaltered index sets, which are used by the next KRBlock to calculate different learned attention biases. KRBlocks have complexity $\mathcal{O}(n_c^2 + n_cn_t)$ where $n_c$ is the number of context points and $n_t$ is the number of test points.

\subsection{$G$-invariant attention bias}\label{subsec:g-invariant-bias}
Recall that the purpose of attention bias is to introduce priors based on domain knowledge that constrain the search space. In spatiotemporal Bayesian inference, this prior often takes the form of a Gaussian Process (GP) prior, specified by a mean function and a covariance kernel whose parameters control the smoothness and range of spatial or temporal dependence between points. For example, since we know that a sick person cannot infect a healthy one 10 miles away, we might use a radial basis function (RBF) kernel operating over their pairwise distance with a lengthscale of 5 feet, which suggests that most transmission happens in close proximity. Priors like this are also translation-invariant, i.e. only the distance between two points matters, rather than their absolute positions. We would like our model to leverage important priors like this and any other relevant $G$-invariant operations such as temporal invariance, rotational invariance, and scale invariance.

Furthermore, when a data distribution exhibits some form of \textit{G}-invariance, this allows the model to exploit structure and extrapolate beyond the training region. This means that these models can be trained on comparatively small input domains and applied to much larger ones at inference time (see \autoref{subsec:gp-bench}). Thus, since training can be calibrated to the available compute, this shifts the primary focus to inference.

BSA-TNP can be configured to be $G$-invariant (see Section~\ref{sec:G-invariance}), as formalized in \autoref{theorem:BSA-inv}. Once again, let $\vd = (\vx, \vs, \vt)$ denote the tuple of fixed effects, spatial locations, and temporal indices, and let $\vd_c$, $\vd_t$ be the context and target index sets, respectively.

\begin{theorem}\label{theorem:BSA-inv}
BSA-TNP is $G$-invariant in $(\vd_c, \vd_t)$ if both the attention bias functions and the embedding functions are $G$-invariant in $(\vd_c, \vd_t)$.
\end{theorem}
Here, the group action $g \in G$ may act on any subset of components in $\vd$, such as scaling household income in \textbf{x}, rotating spatial coordinates in \textbf{s}, translating time indices in \textbf{t}, or any combination thereof. For example, if $G$ represents spatial translation, then the action is defined as $g \act \vd = (\vx, g \act \vs, \vt)$. In terms of BSA-TNP's posterior predictive map, this would be $\pi(\mathbf{x}_c, g\act \mathbf{s}_c, \mathbf{t}_c, \mathbf{f}_c, \mathbf{x}_t, g\act\mathbf{s}_t, \mathbf{t}_t) \mapsto (F(\mathbf{x}_t,\mathbf{s}_t, \mathbf{t}_t)\mid F(\mathbf{x}_c,\mathbf{s}_c,\mathbf{t}_t)=\mathbf{f}_c)$. This means that translating all spatial locations in the input would result in an identical posterior predictive distribution, i.e. it would be translation invariant.

Assuming the embedding function is $G$-invariant, i.e. it does not encode any information about absolute position, scale, time, etc., only the attention bias must also be $G$-invariant in order for BSA-TNP to be $G$-invariant as per \autoref{theorem:BSA-inv}. A proof is provided in \autoref{appendix:BSA-inv-proof}. Thus, in BSA-TNP, we define $G$-invariant attention bias as any combination, $\Phi$, of $G$-invariant functions, $b$, over elements of the index sets, $d_i=(x_i,s_i,t_i)$ and $d_j=(x_j,s_j,t_j)$:
\begin{equation}{\label{eqn:g-inv-bias}}
\mathbf{B}_{ij}=\Phi(b_1(d_i,d_j), b_2(d_i,d_j), \ldots, b_n(d_i, d_j))
\end{equation}
In practice, $\Phi$ is often a weighted linear sum and each $b$ acts on a separate subset of the indices. For example, if $b_s$ represents a spatially translation-invariant function and $b_t$ represents a temporally translation-invariant function, this might be implemented as:
\begin{equation}\label{eqn:g-inv-bias-example}
\mathbf{B}_{ij}=\gamma_s b_s(s_i, s_j) + \gamma_t b_t(t_i, t_j)
\end{equation}
where the weights, $\gamma_s$ and $\gamma_t$ are learnable parameters. Furthermore, $b_s$ and $b_t$ may have learnable parameters and be compositions of other $G$-invariant functions themselves. In fact, for the experiments in this paper, BSA-TNP uses an RBF network for each translation-invariant bias term, since RBF kernels and networks are translation-invariant by definition. Specifically, for a translation-invariant spatial bias, BSA-TNP uses:
\begin{equation}\label{eqn:generic-bias}
b(s_i, s_j)=\sum_{m=1}^M \alpha_m \exp\left(-\beta_m\lVert s_i-s_j\rVert^2\right)
\end{equation}
where $\{\alpha_m,\beta_m\}_{m=1}^M$ are learnable parameters for $M$ basis functions. By default, BSA-TNP uses $M=5$ for space and $M=3$ for time; an ablation of 1, 3, 5, and 10 basis functions showed no benefit for $M>5$ (see \autoref{num-basis-ablation} in \autoref{appendix:extended-results}). This formulation introduces a strong inductive bias with very few parameters. For instance, in the default parameterization of BSA-TNP which has 6 layers, 4 attention heads, 5 basis functions for space, and 3 basis functions for time, this introduces only 384 additional learnable parameters.

While \textit{G}-invariance is most effective for stationary processes, these biases can still prove useful when a process is only partially stationary. For example, when measuring pollution, the distribution is determined by both pollution generation centers and local weather patterns. In this case, it can be advantageous to embed locations \textit{and} use spatial attention bias since locations can act as proxies for cities, which have unique pollution profiles, as well as summarize information about local weather patterns, which affect all locations similarly (see \autoref{subsec:beijing-bench}).

\subsection{Biased Scan Attention (BSA)}\label{subsec:bsa}
In spatiotemporal applications, inference is typically performed over a single large region or a small collection of large regions and the primary computational bottleneck is memory. Fortunately, there has been a wave of innovation in memory-efficient attention variants, e.g.~memory-efficient attention for TPUs \citep{memory_efficient_attention} and Flash Attention 1, 2, and 3 for CUDA and ROCm devices \citep{flash1,flash2,flash3}. Unfortunately, none of these variants support arbitrary bias functions, which means the bias must be omitted or passed as a fully materialized matrix, undermining memory efficiency. Recently, FlexAttention \citep{flex} has provided some support for non-scalar biases, but experiences significant performance regressions during backpropagation (see \autoref{appendix:flex-vs-scan}). Accordingly, we introduce Biased Scan Attention (BSA), a scan-based memory-efficient attention mechanism that supports arbitrary bias functions that take the form of JIT-compiled JAX functions \citep{jax}. BSA largely follows the chunking algorithm defined in Flash Attention 2 \citep{flash2}, but adopts the scan-based tiling and gradient checkpointing of memory efficient attention \citep{memory_efficient_attention}. By using \texttt{jax.lax.scan} with gradient checkpointing, BSA calculates attention scores \textit{and} custom bias terms on the fly with constant memory.

The only requirements of BSA, like its predecessors, are keeping track of the maximum attention score, $m(x)$, the normalization constant, $\ell(x)$, and the unnormalized output, $\mathbf{\tilde{O}}$. For every tile, $j$, \autoref{eqn:bsa_1} is computed. $\mathbf{Q}^{(j)}$ and $\mathbf{K}^{(j)}$ are the queries and keys for that tile and $\mathbf{B}^{(j)}$ is the output of compiled bias functions for that tile. Then for the first tile \autoref{eqn:bsa_2} is computed, and for subsequent tiles, $j>1$, \autoref{eqn:bsa_3} is computed.
\begin{equation}\label{eqn:bsa_1}
\begin{split}
x^{(j)}&=\mathbf{Q}^{(j)}\mathbf{K}^{(j)\intercal}/\sqrt{d_k}+\mathbf{B}^{(j)} \\
f(x)^{(j)}&=e^{x^{(j)}-m(x)^{(j)}\mathbf{1}^\intercal} \\
\end{split}
\end{equation}
\begin{equation}\label{eqn:bsa_2}
\begin{split}
m(x)^{(1)}&= \operatorname{rowmax}(x^{(1)}) \\
\ell(x)^{(1)}&= \operatorname{rowsum}(f(x)^{(1)}) \\
\mathbf{\tilde{O}}^{(1)}&=f(x)^{(1)}\mathbf{V}^{(1)}\\
\end{split}
\end{equation}
\begin{equation}\label{eqn:bsa_3}
\begin{split}
m(x)^{(j>1)}&= \max\left(m(x)^{(j-1)}, \operatorname{rowmax}(x^{(j)})\right) \\
k^{(j>1)}&= e^{m(x)^{(j-1)}-m(x)^{(j)}} \\
\ell(x)^{(j>1)}&= k^{(j>1)}\ell(x)^{(j-1)}+\operatorname{rowsum}(f(x)^{(j)}) \\
\mathbf{\tilde{O}}^{(j>1)}&=k^{(j>1)}\mathbf{\tilde{O}}^{(j-1)} + f(x)^{(j)}\mathbf{V}^{(j)} \\
% \mathbf{O}&=\operatorname{diag}(\ell(x)^{(n)})^{-1}\mathbf{\tilde{O}}^{(n)}
\end{split}
\end{equation}

In short, as each block of size $B$ is processed, three updates occur: (1) the maximum score, $m(x)^{(j)}$, is updated, (2) the normalization constants, $\ell(x)^{(j)}$, are rescaled and updated, and (3) the unnormalized output, $\mathbf{\tilde{O}}^{(j)}$, is rescaled and updated. In the final step, the output is normalized by the final row sums, $\mathbf{O}^{(n)}=\operatorname{diag}(\ell(x)^{(n)})^{-1}\mathbf{\tilde{O}}^{(n)}$.

\section{EXPERIMENTS}\label{sec:experiments}
In this section, we evaluate BSA-TNP, Convolutional Conditional Neural Processes (ConvCNP) \citep{convcnp}, Pseudo Token Translation Equivariant Neural Processes (PT-TE-TNP) \citep{pttnp}, and the original Transformer Neural Process (TNP-D) \citep{tnp} on a collection of synthetic and real-world data. The objective of these experiments, which mostly focus on spatiotemporal applications, is to demonstrate the importance of \textit{G}-invariance, particularly translation invariance in space and time. All models have approximately 500K parameters, except PT-TE-TNP, which has approximately 1.8M. We provide detailed parameterizations in \autoref{appendix:model-param} and model complexity, estimated FLOPs, train times, and inference times in \autoref{appendix:scaling}. We also provide extended results, including runtimes and additional metrics, in \autoref{appendix:extended-results}. All metrics are reported with standard errors over 5 seeds except when noted otherwise. Experiments were run on a single NVIDIA GTX 4090 24GB GPU.

\subsection{2D Gaussian Processes}\label{subsec:gp-bench} % shift, scale, multi-res
\begin{figure*}[t]
    \centering
    \caption{Sphere-stationary GP example comparing geodesic, RBF, and no-bias variants. The geodesic bias remains well aligned under rotation, while the RBF and no-bias variants degrade. The leftmost panel is the task, the center is the prediction, and the right is the uncertainty heatmap where warmer colors indicate greater uncertainty.}
    \label{fig:rot-main}
    \begin{minipage}{0.75\textwidth}
        \centering
        \includegraphics[width=0.15\textwidth]{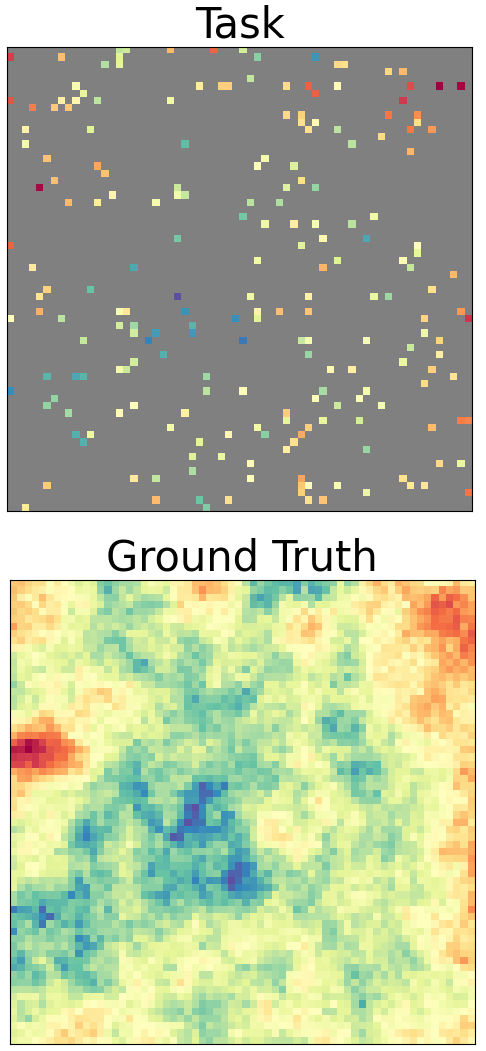}
        \includegraphics[width=0.40\textwidth]{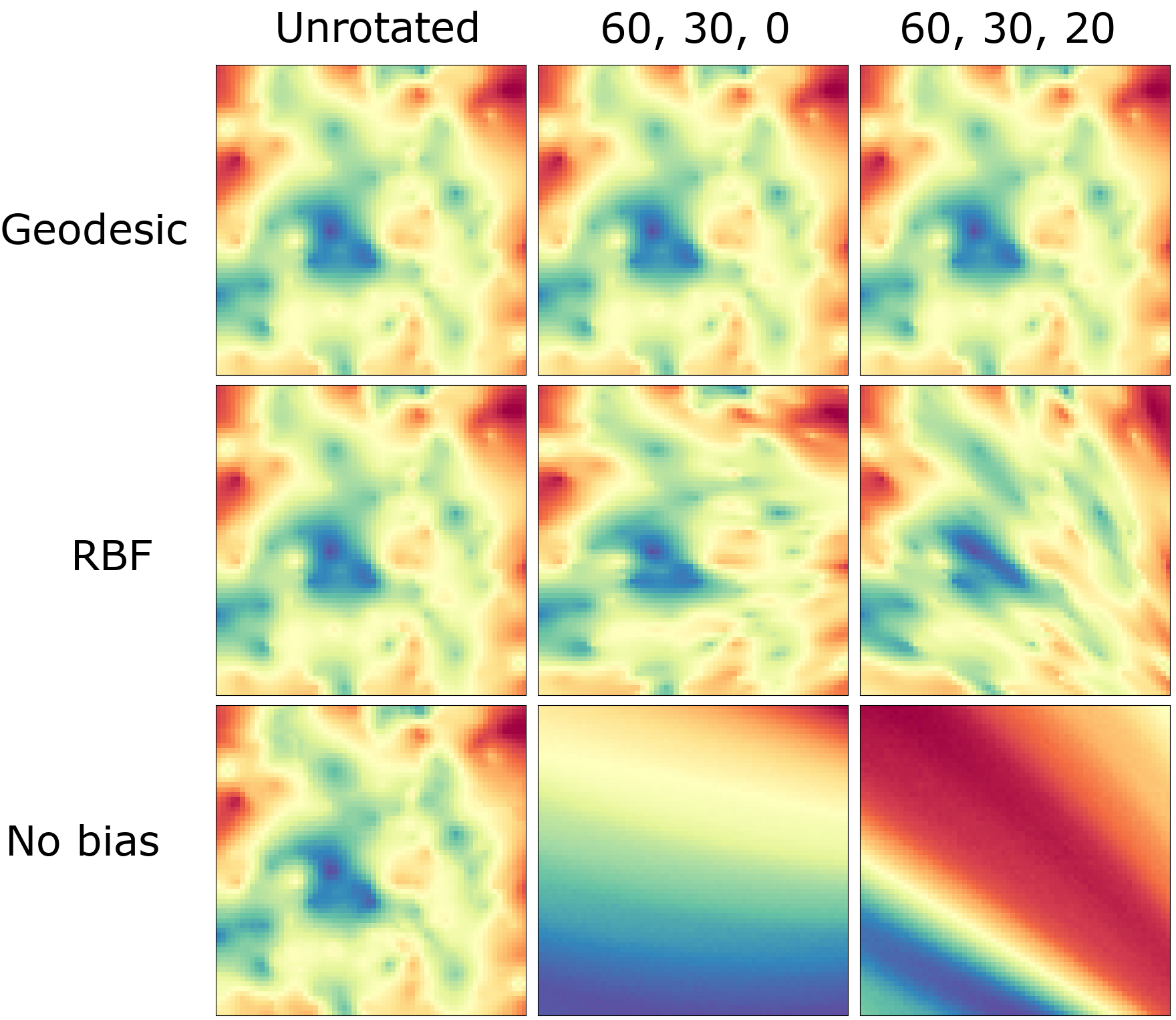}
        \includegraphics[width=0.40\textwidth]{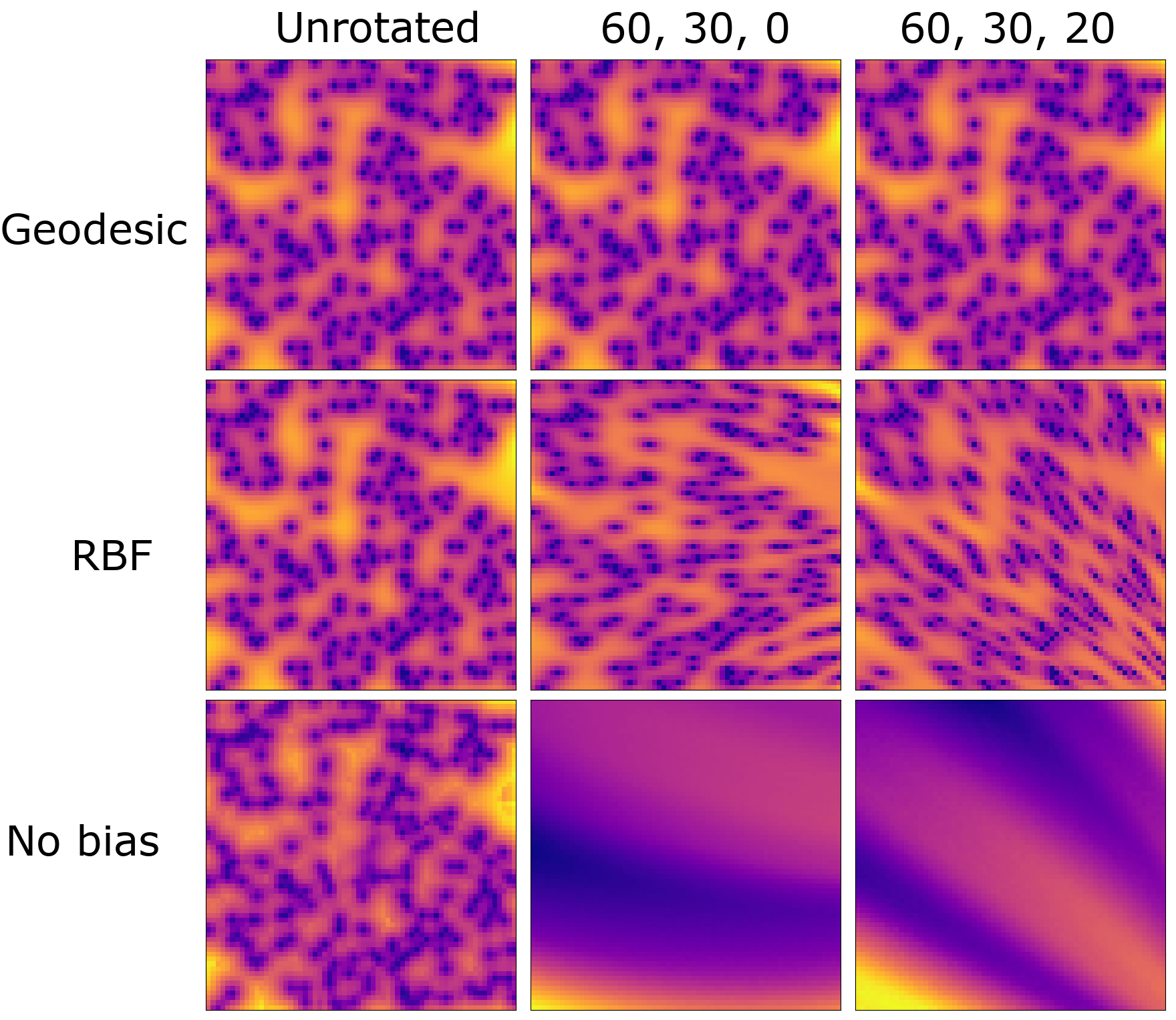}
    \end{minipage}
\end{figure*}
For this benchmark, we create tasks from 2D GP samples with an RBF kernel and lengthscale sampled from $\text{Beta}(3.0, 7.0)$, which has a mean and median of $\approx$ 0.3. This is a more challenging benchmark than is typically used and allows for greater differentiation among models since more than 50\% of lengthscales fall below 0.3 and less than 10\% lie above 0.5. (In practice, we found that even less powerful NPs could easily learn lengthscales above 0.4-0.5). We provide an extensive description of the 2D GP tasks in \autoref{appendix:exp}.

In order to evaluate the generalization of various NPs, we evaluate them under four different scenarios: (1) tasks sampled from the training domain, (2) tasks sampled from the training domain shifted by 10 units, (3) tasks where the domain has doubled along each axis, and (4) multiscale tasks where the model receives data at two different resolutions simultaneously. Scenario (4) is particularly relevant because it mimics a situation where satellite imagery is available at two different levels of granularity and must be synthesized. Multiresolution training is also useful under limited compute: by providing the model with high-resolution data for a target region and coarse-resolution data for the surrounding area, the model can focus on the target region while still incorporating important contextual information. In this benchmark, using the coarser resolution for the surrounding area results in a $\approx$78\% reduction in the total number of pixels. \autoref{appendix:multires-gp} shows an example of a multiresolution task. Because ConvCNP requires a fixed grid and TNP-D fails under shifting and scaling, we compare only PT-TE-TNP and BSA-TNP on (4).

BSA-TNP not only outperforms the other methods on tasks from the original training domain, but also suffers no performance degradation under the out-of-distribution tasks, i.e. shifting and scaling. TNP-D fails because it does not exhibit any form of translation invariance. ConvCNP oversmooths output due to the induced latent grid, and PT-TE-TNP struggles due to the bottleneck imposed by the inducing points. We provide an extended comparison to the full TE-TNP as well as a variant of PT-TE-TNP with more inducing points in \autoref{2d-gp} of \autoref{appendix:extended-results}.
\begin{table}[t]
\centering
\setlength{\tabcolsep}{3pt}
\caption{NLL on 2D GP tasks.}
\label{tab:gp-all}
\resizebox{\columnwidth}{!}{%
\begin{tabular}{@{}lccc@{}}
\toprule
Model/Domain & Original & Shifted & Scaled \\
\midrule
ConvCNP & $-0.22\pm0.00$ & $-0.22\pm0.00$ & $-0.23\pm0.01$ \\
TNP-D   & $-0.27\pm0.01$ & $20.19\pm11.17$ & $0.67\pm0.18$  \\
PT-TE-TNP & $0.40\pm0.02$ & $0.40\pm0.02$ & $1.45\pm0.08$ \\
BSA-TNP & $\mathbf{-0.32\pm0.00}$ & $\mathbf{-0.32\pm0.00}$ & $\mathbf{-0.28\pm0.01}$ \\
\bottomrule
\end{tabular}%
}
\end{table}

\begin{table}[t]
\centering
\setlength{\tabcolsep}{3pt}
\caption{Multiscale GP tasks trained on two scales simultaneously. Half of the context points come from $[-0.5,0.5]$ and the other half from $[-1.5,1.5]$, simulating a high resolution target zone with coarser observations from the surrounding area.}
\label{tab:multiscale}
\resizebox{\columnwidth}{!}{%
\begin{tabular}{@{}lcccc@{}}
\toprule
Model & NLL & MAE & RMSE & CVG@95\% \\
\midrule
PT-TE-TNP & $0.18\pm0.01$ & $0.32\pm0.00$ & $0.48\pm0.00$ & $\mathbf{0.95\pm0.00}$ \\
BSA-TNP   & $\mathbf{-0.24\pm0.00}$ & $\mathbf{0.25\pm0.00}$ & $\mathbf{0.40\pm0.00}$ & $\mathbf{0.95\pm0.00}$ \\
\bottomrule
\end{tabular}%
}
\end{table}

\subsection{Spherical Gaussian Processes}\label{subsec:gp-spherical} % shift, scale, multi-res
We also evaluate BSA-TNP on spherical GPs, where Euclidean translations are only a local approximation to the correct symmetry. We compare an unbiased SA-TNP, BSA-TNP with the standard RBF bias, and BSA-TNP with an SO(3)-invariant geodesic bias. Tasks are sampled on a spherical patch and then evaluated under increasingly large rotations; full details are given in \autoref{appendix:exp}. \autoref{tab:rot-main} shows that all three variants perform similarly without rotation, but only the geodesic bias remains stable under rotated test tasks. The RBF bias degrades as the rotation becomes more pronounced, while the unbiased model breaks down. \autoref{fig:rot-main} visualizes these effects.

\begin{table}
\centering
\setlength{\tabcolsep}{3pt}
\caption{Spherical Gaussian Processes.}
\label{tab:rot-main}
\resizebox{\columnwidth}{!}{
\begin{tabular}{@{}lccc@{}}
\toprule
Model & Unrotated & $60^\circ, 30^\circ, 0^\circ$ & $60^\circ, 30^\circ, 20^\circ$ \\
\midrule
SA-TNP (No bias) & $\mathbf{-0.01\pm0.00}$ & $16.41\pm13.35$ & $11.69\pm7.47$ \\
BSA-TNP (RBF) & $\mathbf{-0.01\pm0.00}$ & $0.05\pm0.00$ & $0.08\pm0.00$ \\
BSA-TNP (Geodesic) & $\mathbf{-0.01\pm0.00}$ & $\mathbf{-0.01\pm0.00}$ & $\mathbf{-0.01\pm0.00}$ \\
\bottomrule
\end{tabular}
}
\end{table}

\begin{figure*}[t]
  \centering
  \includegraphics[width=0.7\linewidth]{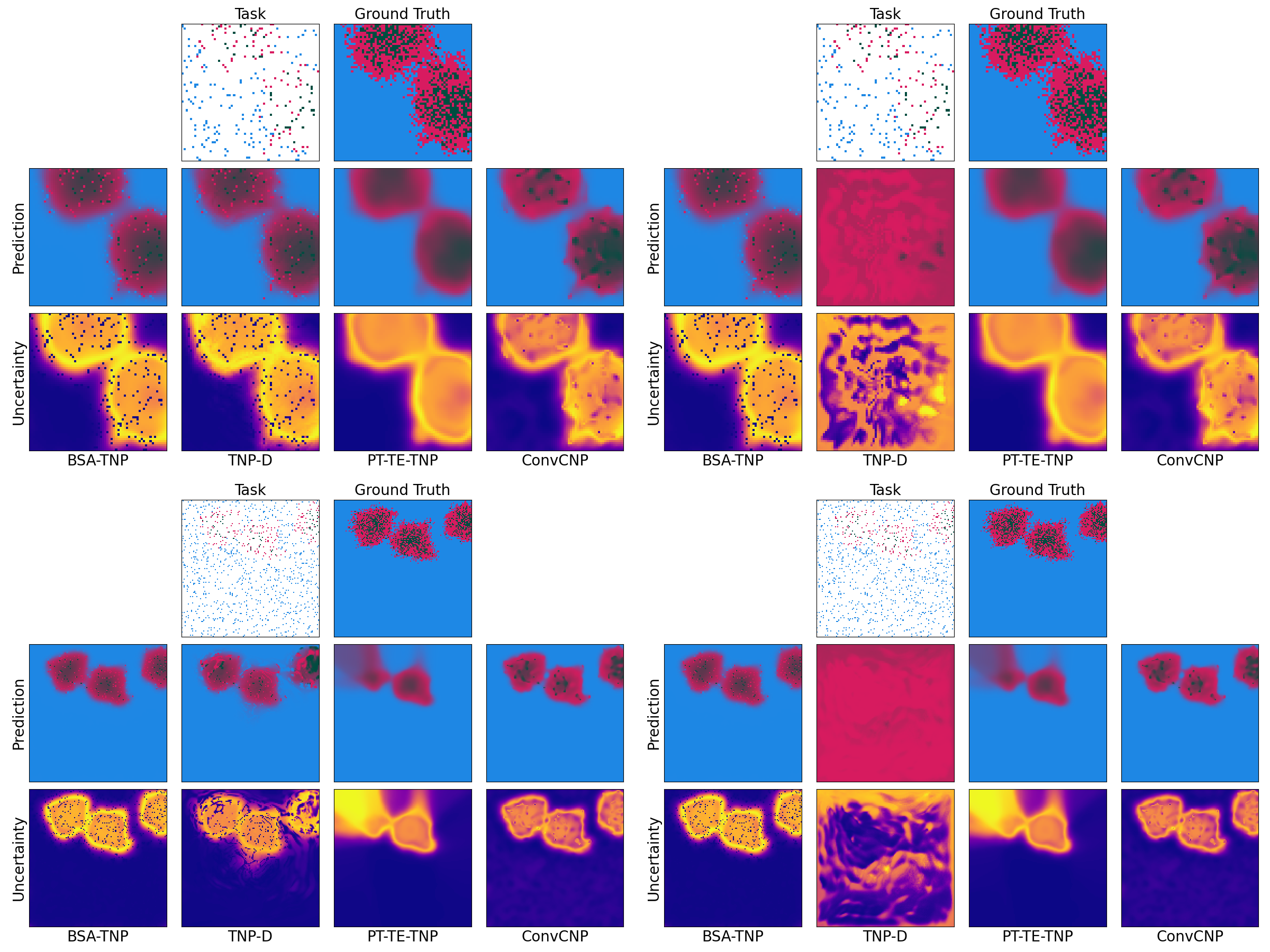}
  \caption{Susceptible-Infected-Recovered (SIR) tasks on the original (upper left), shifted (upper right), scaled (bottom left), and shifted$+$scaled domains. BSA-TNP preserves performance under both shifts and scaling, while TNP-D and PT-TE-TNP degrade out of distribution; ConvCNP extrapolates but oversmooths.}
  \label{fig:sir}
\end{figure*}

\subsection{Epidemiology} % SIR
To evaluate performance on an epidemiological application, we generate tasks from the Susceptible-Infected-Recovered (SIR) model, which models the spread of infectious disease outbreaks. It is governed by an infection rate, $\beta$, a recovery rate, $\gamma$, and the number of initial infections, $\omega$. We sample $\beta\sim\operatorname{Beta}(2, 8)$, $\gamma\sim\operatorname{InvGamma}(5, 0.4)$, and $\omega\sim\operatorname{randint}(1,5)$. The infection rate, $\beta$, is decreased as an inverse function of distance from the infected individual using a convolutional kernel with width 9 (see \autoref{appendix:exp} for more details). In expectation, this parameter setting corresponds to an infection rate of 20\% upon exposure and a 10-day recovery period (similar to COVID-19). We train all models on 64x64 images and use the same task distribution as the 2D GP benchmark. A notable difference here is that we train and evaluate models with context points included in the test points. The reason for this is that, unlike infections, recoveries are independent of one another. If the model has been told that a particular individual recovered, it should be able to reproduce that in the output, which is not the case for models that smooth inputs, e.g.~PT-TE-TNP and ConvCNP. BSA-TNP matches the performance of TNP-D on the original domain and outperforms all other methods on the out-of-distribution tasks.

\begin{table}[t]
\centering
\setlength{\tabcolsep}{3pt}
\caption{NLL on SIR tasks.}
\label{tab:sir}
\resizebox{\columnwidth}{!}{%
\begin{tabular}{@{}lccc@{}}
\toprule
Model/Domain & Original & Shifted & Scaled \\
\midrule
ConvCNP   & $0.24\pm0.00$ & $0.24\pm0.00$ & $0.21\pm0.00$ \\
TNP-D     & $\mathbf{0.19\pm0.00}$ & $3.00\pm0.37$ & $0.24\pm0.01$ \\
PT-TE-TNP & $0.27\pm0.00$ & $0.27\pm0.00$ & $0.44\pm0.02$ \\
BSA-TNP   & $\mathbf{0.19\pm0.00}$ & $\mathbf{0.19\pm0.00}$ & $\mathbf{0.18\pm0.00}$ \\
\bottomrule
\end{tabular}%
}
\end{table}
As this benchmark relies on a standard grid, this makes it amenable to clean scaling tests. We demonstrate the scalability of BSA-TNP on $128\times128$, $256\times256$, $512\times512$, and $1024\times1024$ resolutions in \autoref{tab:sir-scale} of \autoref{appendix:extended-results}.

\subsection{Climate} % ERA5
\begin{figure}[t]  % [t], [b], or [h] for placement
  \centering
  \includegraphics[width=0.975\linewidth]{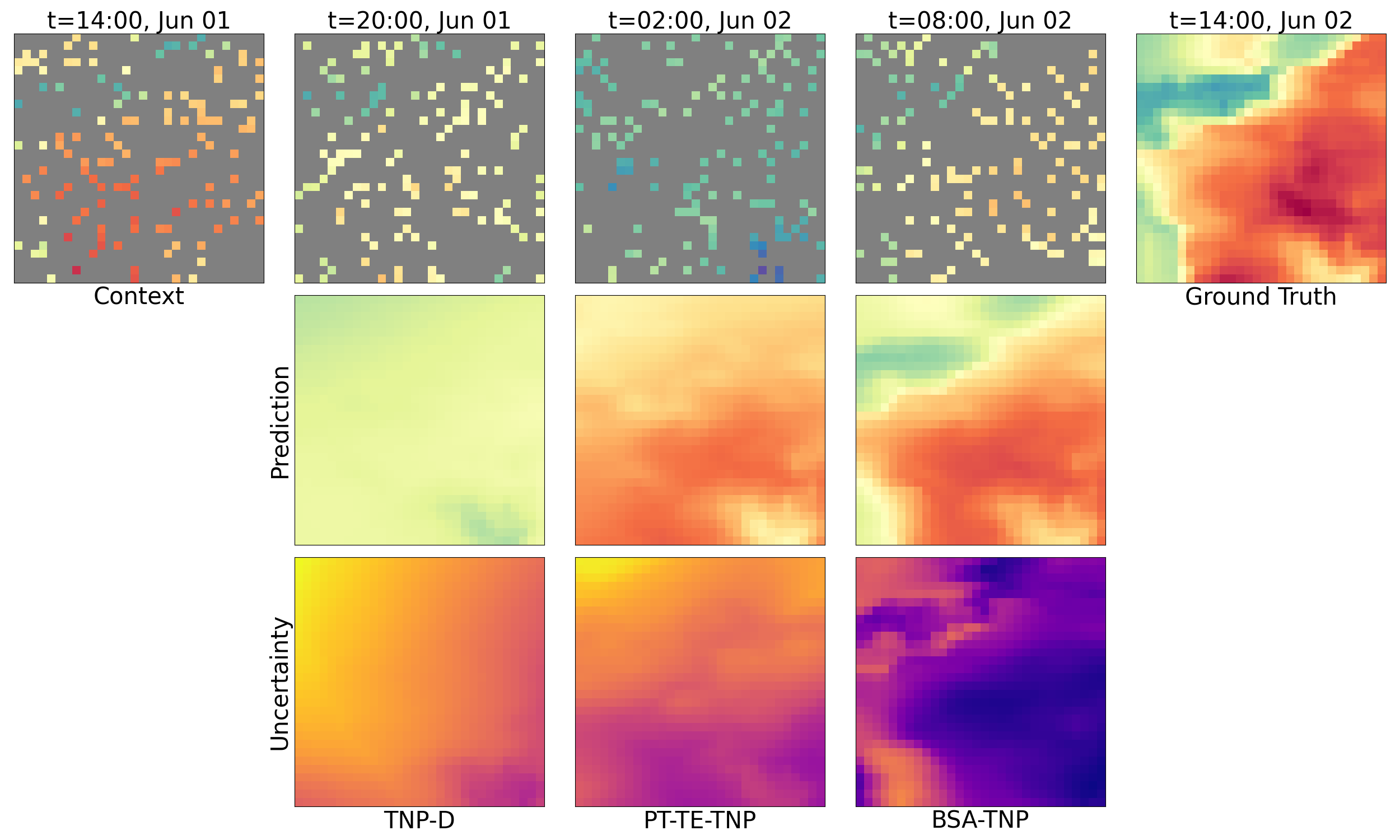}
  \caption{ERA5 forecast example on an out-of-region northern Europe task.}% BSA-TNP tracks the target field more closely than TNP-D and PT-TE-TNP, and correctly leverages the 24-hour context frame.}
  \label{fig:era5}
\end{figure}
For climate, we largely follow the setup in \cite{pttnp} with ERA5 \citep{era5} surface air temperatures. To test learning over time and generalization across space, we train models on tasks from central Europe $[42^\circ\mathrm{N},\,53^\circ\mathrm{N}]\times[8^\circ\mathrm{E},\,28^\circ\mathrm{E}]$ and test them on tasks from northern Europe $[53^\circ\mathrm{N},\,62^\circ\mathrm{N}]\times[8^\circ\mathrm{E},\,28^\circ\mathrm{E}]$ and western Europe $[42^\circ\mathrm{N},\,53^\circ\mathrm{N}]\times[4^\circ\mathrm{W},\,8^\circ\mathrm{E}]$ over random 30-hour segments of time in 6-hour increments. When testing on northern Europe, we use western Europe as a validation set and select the best model for testing and vice versa when testing on western Europe. Unlike \cite{pttnp}, which inpaints several frames over time, we change the benchmark to forecast the weather in the next 6 hours. We also add ``hour of day'' to the feature set and increase the resolution from $0.5^\circ$ to $0.25^\circ$. % The task consists of 4 context frames and 1 test frame, each separated by 6 hours. The inputs we use are \textbf{x}=(elevation, hour of day), \textbf{s}=(latitude, longitude), \textbf{t}=time, \textbf{f}=surface temperature.
More details on the experimental setup can be found in \autoref{appendix:exp}. \autoref{fig:era5} provides a visualization of an example task, and the results in \autoref{tab:era5-cnw} and \autoref{tab:era5-cwn} show that BSA-TNP outperforms its competitors. ConvCNP was excluded from this benchmark because it does not natively handle space and time or extra covariates.
\begin{table}[t]
\centering
\setlength{\tabcolsep}{3pt} 
\caption{ERA5 CNW - Training on central, validating on northern, and testing on western Europe.}
\label{tab:era5-cnw}
\resizebox{\columnwidth}{!}{%
\begin{tabular}{@{}lcccc@{}}
\toprule
Model & NLL & MAE & RMSE & CVG@95\% \\
\midrule
TNP-D     & $0.38\pm0.06$ & $0.23\pm0.01$ & $0.29\pm0.01$ & $0.84\pm0.01$ \\
PT-TE-TNP & $0.17\pm0.02$ & $0.22\pm0.00$ & $0.28\pm0.01$ & $0.90\pm0.01$ \\
BSA-TNP   & $\mathbf{-0.07\pm0.04}$ & $\mathbf{0.18\pm0.01}$ & $\mathbf{0.23\pm0.01}$ & $\mathbf{0.91\pm0.01}$ \\
\bottomrule
\end{tabular}%
}
\end{table}
\begin{table}[t]
\centering
\setlength{\tabcolsep}{3pt} % tighten column padding
\caption{ERA5 CWN - Training on central, validating on western, and testing on northern Europe.}
\label{tab:era5-cwn}
\resizebox{\columnwidth}{!}{%
\begin{tabular}{@{}lcccc@{}}
\toprule
Model & NLL & MAE & RMSE & CVG@95\% \\
\midrule
TNP-D     & $0.67\pm0.06$ & $0.25\pm0.01$ & $0.32\pm0.01$ & $0.78\pm0.01$ \\
PT-TE-TNP & $0.13\pm0.02$ & $0.21\pm0.00$ & $0.27\pm0.00$ & $0.90\pm0.01$ \\
BSA-TNP   & $\mathbf{-0.13\pm0.04}$ & $\mathbf{0.17\pm0.01}$ & $\mathbf{0.22\pm0.01}$ & $\mathbf{0.91\pm0.01}$ \\
\bottomrule
\end{tabular}%
}
\end{table}
We also compared BSA-TNP to a non-amortized sparse variational GP (SVGP) baseline that fits a fresh weather-kernel GP to each ERA5 task using elevation, hour of day, latitude, longitude, and time as inputs. After tuning the SVGP on held-out tasks, BSA-TNP still delivered substantially better accuracy (NLL $-0.01$ vs.~$1.31$, RMSE $0.23$ vs.~$0.34$) and calibration ($0.90$ vs.~$0.71$ CVG@95\%), while avoiding roughly 55 seconds of per-task optimization; full results are given in \autoref{tab:era5-bsa-vs-svgp-results} of \autoref{appendix:extended-results}.

\subsection{Partially Stationary Processes}\label{subsec:beijing-bench} % Beijing Air Quality + synthetic
To test whether our attention bias remains useful when strict \textit{G}-invariance is only approximate, we evaluate on two complementary partially stationary benchmarks: one real-world and one synthetic. The first is the Beijing Multi-Site Air Quality dataset from the UCI repository \citep{beijing}. After merging station coordinates and dropping rows with missing values, we use calendar and weather covariates, latitude/longitude, and elapsed time to forecast PM2.5 and PM10 across the 12 sites; wind direction is one-hot encoded. Following a chronological 80/10/10 train-validation-test split, each task uses an intended 48-hour history for each site, padded by 4 extra hours per site to absorb gaps, to forecast an intended 6-hour horizon padded by 1 extra hour per site. The second is a synthetic Gneiting GP benchmark in which each task is sampled on an irregular 2D spatial layout over five consecutive time points, with random subsets of locations observed at the first four times and the fifth used as the forecast target. Its covariance follows the nonseparable Gneiting kernel, while its mean combines a fixed terrain field, periodic temporal components, a linear trend, and a moving spatial hotspot, creating smooth local structure without full translation invariance. \autoref{tab:beijing-air-quality} and \autoref{gneiting-gp} show that the bias remains helpful in both settings: bias-only performs best on the real Beijing task, while combining bias and embeddings works best on the synthetic benchmark. We hypothesize that bias-only performs best in Beijing through implicit regularization, i.e. embedding locations can more easily overfit the training distribution.
\begin{table}
\caption{Beijing Air Quality}
\label{tab:beijing-air-quality}
\resizebox{\columnwidth}{!}{%
\begin{tabular}{lccc}
\toprule
Name & NLL & RMSE & CVG@95\% \\
\midrule
TNP-D & $-1.33\pm0.15$ & $0.06\pm0.00$ & $0.88\pm0.01$ \\
PT-TE-TNP & $-1.60\pm0.04$ & $0.06\pm0.00$ & $\mathbf{0.92\pm0.00}$ \\
BSA-TNP (embed) & $-1.75\pm0.04$ & $0.06\pm0.00$ & $0.90\pm0.00$ \\
BSA-TNP (bias) & $\mathbf{-1.82\pm0.03}$ & $\mathbf{0.05\pm0.00}$ & $0.91\pm0.00$ \\
BSA-TNP (embed+bias) & $-1.75\pm0.04$ & $\mathbf{0.05\pm0.00}$ & $0.90\pm0.01$ \\
\bottomrule
\end{tabular}
}
\end{table}

\begin{table}
\caption{Gneiting GP}
\label{gneiting-gp}
\resizebox{\columnwidth}{!}{%
\begin{tabular}{lccc}
\toprule
Name & NLL & RMSE & CVG@95\% \\
\midrule
TNP-D & $1.45\pm0.01$ & $1.03\pm0.01$ & $0.94\pm0.00$ \\
PT-TE-TNP & $1.46\pm0.00$ & $1.03\pm0.00$ & $\mathbf{0.95\pm0.00}$ \\
BSA-TNP (embed) & $1.46\pm0.00$ & $1.04\pm0.00$ & $0.94\pm0.00$ \\
BSA-TNP (bias) & $1.39\pm0.00$ & $0.99\pm0.00$ & $0.94\pm0.00$ \\
BSA-TNP (embed+bias) & $\mathbf{1.38\pm0.00}$ & $\mathbf{0.98\pm0.00}$ & $\mathbf{0.95\pm0.00}$ \\
\bottomrule
\end{tabular}
}
\end{table}

\section{LIMITATIONS AND FUTURE WORK}
While most NP tasks assume sparse observations and dense predictions, some tasks have dense observations and sparse predictions, e.g.~video prediction tasks. In this case, BSA-TNP still has quadratic complexity in the number of context points, making training and inference over a collection of dense frames computationally burdensome. In future work, we hope to extend BSA-TNP to incorporate frame summaries from models like I-JEPA \citep{ijepa} to overcome this limitation.

\section{CONCLUSION}\label{sec:conclusion}
In this work, we introduce BSA-TNP, which represents a significant step forward in NP design by combining simplicity, scalability, and extensibility with KRBlocks and \textit{G}-invariant attention biases. Its ability to capture complex spatiotemporal dependencies with minimal overhead not only delivers state-of-the-art accuracy, but also enables inference at scale. Because of its broad applicability, we hope it will form a foundational tool for scientists who work with spatiotemporal phenomena and accelerate research in critical domains like climate, epidemiology, and robotics.

\subsubsection*{Acknowledgments}
E.S. and J.N. acknowledge support in part by the AI2050 program at Schmidt Sciences (Grant [G-22-64476]). D.J. acknowledges his Google DeepMind scholarship. 
\bibliographystyle{plainnat}
\bibliography{main}
\clearpage
\appendix
% Ensure section labels created in the appendix are typed as appendices.
\crefalias{section}{appendix}
\thispagestyle{empty}

\onecolumn
\newpage
\section*{Checklist}
\begin{enumerate}
  \item For all models and algorithms presented, check if you include:
  \begin{enumerate}
    \item A clear description of the mathematical setting, assumptions, algorithm, and/or model. \textbf{Yes}
    \item An analysis of the properties and complexity (time, space, sample size) of any algorithm. \textbf{Yes}
    \item (Optional) Anonymized source code, with specification of all dependencies, including external libraries. \textbf{Yes}
  \end{enumerate}
  \item For any theoretical claim, check if you include:
  \begin{enumerate}
    \item Statements of the full set of assumptions of all theoretical results. \textbf{Yes}
    \item Complete proofs of all theoretical results. \textbf{Yes}
    \item Clear explanations of any assumptions. \textbf{Yes}
  \end{enumerate}

  \item For all figures and tables that present empirical results, check if you include:
  \begin{enumerate}
    \item The code, data, and instructions needed to reproduce the main experimental results (either in the supplemental material or as a URL). \textbf{Yes}
    \item All the training details (e.g., data splits, hyperparameters, how they were chosen). \textbf{Yes}
    \item A clear definition of the specific measure or statistics and error bars (e.g., with respect to the random seed after running experiments multiple times). \textbf{Yes}
    \item A description of the computing infrastructure used. (e.g., type of GPUs, internal cluster, or cloud provider). \textbf{Yes}
  \end{enumerate}

  \item If you are using existing assets (e.g., code, data, models) or curating/releasing new assets, check if you include:
  \begin{enumerate}
    \item Citations of the creator If your work uses existing assets. \textbf{Yes}
    \item The license information of the assets, if applicable. \textbf{Yes}
    \item New assets either in the supplemental material or as a URL, if applicable. \textbf{Yes}
    \item Information about consent from data providers/curators. \textbf{Not Applicable}
    \item Discussion of sensible content if applicable, e.g., personally identifiable information or offensive content. \textbf{Not Applicable}
  \end{enumerate}

  \item If you used crowdsourcing or conducted research with human subjects, check if you include:
  \begin{enumerate}
    \item The full text of instructions given to participants and screenshots. \textbf{Not Applicable}
    \item Descriptions of potential participant risks, with links to Institutional Review Board (IRB) approvals if applicable. \textbf{Not Applicable}
    \item The estimated hourly wage paid to participants and the total amount spent on participant compensation. \textbf{Not Applicable}
  \end{enumerate}
\end{enumerate}

\clearpage
\appendix
\section{Model Parameterizations}\label{appendix:model-param}
\paragraph{ConvCNP:} We use the ``off-grid'' version of ConvCNP since we train and test on off-grid observations for many benchmarks. We use an induced density of 16 points per unit, which corresponds to 64 units per axis for most benchmarks that are trained on $[-2, 2]^2$. This implies that in a 64x64 image, each pixel has its own grid point. We use ConvDeepSets for encoding and decoding to the latent grid and a stack of 8 ConvCNPNet blocks with kernel size (9, 9) and feature dimension of 128. The prediction head consists of a 4-layer MLP with 3 layers with 128 hidden units and a final layer projecting to the output dimension. This parameterization has $\approx 507K$ parameters, which varies slightly depending on the dimension of inputs and outputs.
\paragraph{TNP-D:} We use the standard formulation defined in \cite{tnp}, which consists of a stack of transformer encoder layers with a special mask to prevent context and test points from attending to test points. We make a slight departure for the Beijing Multi-City Air Quality and Gneiting GP benchmarks where we modify the transformer blocks to use pre-normalization. We found that TNP-D struggled to learn without this, even though it was not part of the original design. TNP-D has a 3-layer embedding MLP that consists of 256, 128, and 64 units, respectively. It uses 6 transformer encoder blocks each with 4 attention heads. The token embedding dimension is 64 throughout except when performing attention, when we upscale the queries, keys, and values to 128. Feedforward layers consist of two hidden layers with 256 and 64 units, respectively. The prediction head is identical to the feedforward layers except it has an additional layer projecting to the output dimension. This parameterization results in $\approx$477K parameters, which varies slightly depending on the dimension of inputs and outputs. We parameterize BSA-TNP in an almost identical fashion to maintain as much parity as possible.
\paragraph{PT-TE-TNP:} We follow the default \href{https://github.com/cambridge-mlg/tetnp/blob/main/experiments/configs/models/teist.yml}{parameterization from the GitHub repository} associated with \cite{pttnp}. We use the IST-based architecture, as the authors noted this performed better in their tests. The model uses a token dimension of 128, 8 attention heads, 5 stacked PT-TE-IST blocks, and 32 pseudo-tokens. This parameterization results in $\approx$ 1.85M parameters, which varies slightly depending on the dimension of inputs and outputs.
\paragraph{BSA-TNP:} We use 6 KRBlocks, 4 attention heads, and a token embedding of 64. We upscale this to 128 when performing attention, similar to TNP-D. Also like TNP-D, we use a 3-layer embedding MLP that consists of 256, 128, and 64 units, respectively. Feedforward layers consist of two hidden layers with 256 and 64 units, respectively. The prediction head is identical to the feedforward layers except it has an additional layer projecting to the output dimension. For benchmarks that test translation invariance, we do not embed space or time features, and pass these only to the bias functions. Furthermore, for spatial and temporal bias, we use 5 and 3 basis functions, respectively, per attention head per layer. This parameterization results in $\approx$ 478K parameters, which varies slightly depending on the dimension of inputs and outputs.

\newpage
\section{Extended Results}\label{appendix:extended-results}
Here we provide full results, including runtimes and extra metrics that may not have fit in the main paper.
% GPs
\begin{table}[H]
\centering
\caption{2D GP extended results. PT-TE-TNP (M=128) and TE-TNP took approximately 4 hours and 8 hours per run, respectively. This is because the full attention matrices are passed through an MLP in each layer. If there are $L$ locations and the MLP is only 2 layers with a small hidden dimension, e.g. $H=128$, this makes the translation-equivariant attention calculation $O(L^2\cdot 2H)=O(256L^2)$, i.e. it is 256x slower than regular attention per attention calculation (for each head in each layer). The gains, if any, from using these more computationally intensive models were minimal, and so we use the default implementation of PT-TE-TNP with 32 latents in all other benchmarks.}
\label{2d-gp}
\begin{tabular}{lccccc}
\toprule
Name & NLL & MAE & RMSE & CVG@95\% & Runtime (min) \\
\midrule
ConvCNP & $-0.22\pm0.00$ & $0.23\pm0.00$ & $0.32\pm0.00$ & $\mathbf{0.95\pm0.00}$ & $125.21\pm1.88$ \\
TNP-D & $-0.27\pm0.01$ & $0.22\pm0.00$ & $0.31\pm0.00$ & $\mathbf{0.95\pm0.00}$ & $65.92\pm0.01$ \\
TE-TNP & $0.27\pm0.01$ & $0.29\pm0.00$ & $0.40\pm0.00$ & $0.92\pm0.00$ & $488.56\pm0.11$ \\
PT-TE-TNP (M=32) & $0.40\pm0.02$ & $0.33\pm0.00$ & $0.44\pm0.01$ & $\mathbf{0.95\pm0.00}$ & $74.30\pm0.12$ \\
PT-TE-TNP (M=128) & $0.16\pm0.06$ & $0.28\pm0.01$ & $0.39\pm0.01$ & $\mathbf{0.95\pm0.00}$ & $227.92\pm0.17$ \\
BSA-TNP & $\mathbf{-0.32\pm0.00}$ & $\mathbf{0.21\pm0.00}$ & $\mathbf{0.30\pm0.00}$ & $\mathbf{0.95\pm0.00}$ & $\mathbf{40.79\pm0.12}$ \\
\bottomrule
\end{tabular}
\end{table}

\begin{table}[H]
\centering
\caption{2D GP evaluated on tasks shifted by 10 units relative to the training domain.}
\label{2d-gp-shifted-10}
\begin{tabular}{lccccc}
\toprule
Name & NLL & MAE & RMSE & CVG@95\% & Runtime (min) \\
\midrule
ConvCNP & $-0.22\pm0.00$ & $0.23\pm0.00$ & $0.32\pm0.00$ & $\mathbf{0.95\pm0.00}$ & $3.00\pm0.01$ \\
TNP-D & $20.19\pm11.17$ & $0.77\pm0.01$ & $0.96\pm0.02$ & $0.47\pm0.15$ & $1.30\pm0.00$ \\
PT-TE-TNP & $0.40\pm0.02$ & $0.33\pm0.00$ & $0.44\pm0.01$ & $\mathbf{0.95\pm0.00}$ & $1.13\pm0.00$ \\
BSA-TNP & $\mathbf{-0.32\pm0.00}$ & $\mathbf{0.21\pm0.00}$ & $\mathbf{0.30\pm0.00}$ & $\mathbf{0.95\pm0.00}$ & $\mathbf{0.90\pm0.00}$ \\
\bottomrule
\end{tabular}
\end{table}

\begin{table}[H]
\centering
\caption{2D GP evaluated on tasks whose spatial domain is doubled along each axis relative to training.}
\label{2d-gp-scaled-2x}
\begin{tabular}{lccccc}
\toprule
Name & NLL & MAE & RMSE & CVG@95\% & Runtime (min) \\
\midrule
ConvCNP & $-0.23\pm0.01$ & $0.23\pm0.00$ & $0.32\pm0.00$ & $0.93\pm0.00$ & $4.46\pm0.01$ \\
TNP-D & $0.67\pm0.18$ & $0.34\pm0.02$ & $0.47\pm0.03$ & $0.90\pm0.02$ & $3.69\pm0.00$ \\
PT-TE-TNP & $1.45\pm0.08$ & $0.72\pm0.01$ & $0.96\pm0.01$ & $0.87\pm0.02$ & $\mathbf{2.03\pm0.01}$ \\
BSA-TNP & $\mathbf{-0.28\pm0.01}$ & $\mathbf{0.21\pm0.00}$ & $\mathbf{0.30\pm0.00}$ & $\mathbf{0.94\pm0.01}$ & $2.90\pm0.03$ \\
\bottomrule
\end{tabular}
\end{table}

\begin{table}[H]
\centering
\caption{Multiscale 2D GP in which each task combines fine-resolution targets with coarser surrounding context. As in the main text, only PT-TE-TNP and BSA-TNP are compared because ConvCNP requires a fixed grid and TNP-D does not handle this out-of-distribution scaling regime well.}
\label{multiscale-gp}
\begin{tabular}{lccccc}
\toprule
Name & NLL & MAE & RMSE & CVG@95\% & Runtime (min) \\
\midrule
PT-TE-TNP & $0.18\pm0.01$ & $0.32\pm0.00$ & $0.48\pm0.00$ & $\mathbf{0.95\pm0.00}$ & $57.13\pm0.11$ \\
BSA-TNP & $\mathbf{-0.24\pm0.00}$ & $\mathbf{0.25\pm0.00}$ & $\mathbf{0.40\pm0.00}$ & $\mathbf{0.95\pm0.00}$ & $\mathbf{17.79\pm0.04}$ \\
\bottomrule
\end{tabular}
\end{table}

% GPs rotated
\begin{table}[H]
\centering
\caption{Spherical GP without rotation, included as the in-distribution reference for the rotational-invariance benchmark.}
\label{gp-rotated-0-0-0}
\begin{tabular}{lccccc}
\toprule
Name & NLL & MAE & RMSE & CVG@95\% & Runtime (min) \\
\midrule
SA-TNP (No bias) & $\mathbf{-0.01\pm0.00}$ & $\mathbf{0.20\pm0.00}$ & $\mathbf{0.26\pm0.00}$ & $\mathbf{0.95\pm0.00}$ & $\mathbf{33.90\pm0.01}$ \\
BSA-TNP (RBF) & $\mathbf{-0.01\pm0.00}$ & $\mathbf{0.20\pm0.00}$ & $\mathbf{0.26\pm0.00}$ & $\mathbf{0.95\pm0.00}$ & $41.04\pm0.13$ \\
BSA-TNP (Geodesic) & $\mathbf{-0.01\pm0.00}$ & $\mathbf{0.20\pm0.00}$ & $\mathbf{0.26\pm0.00}$ & $\mathbf{0.95\pm0.00}$ & $40.95\pm0.01$ \\
\bottomrule
\end{tabular}
\end{table}

\begin{table}[H]
\centering
\caption{Spherical GP rotated by $60^\circ$ north and $30^\circ$ east. The geodesic-bias variant remains aligned under this rotation while the Euclidean-bias and no-bias variants degrade.}
\label{gp-rotated-60-30-0}
\begin{tabular}{lccccc}
\toprule
Name & NLL & MAE & RMSE & CVG@95\% & Runtime (min) \\
\midrule
SA-TNP (No bias) & $16.41\pm13.35$ & $0.73\pm0.01$ & $0.91\pm0.01$ & $0.62\pm0.16$ & $\mathbf{0.79\pm0.00}$ \\
BSA-TNP (RBF) & $0.05\pm0.00$ & $0.21\pm0.00$ & $0.27\pm0.00$ & $0.92\pm0.00$ & $0.91\pm0.00$ \\
BSA-TNP (Geodesic) & $\mathbf{-0.01\pm0.00}$ & $\mathbf{0.20\pm0.00}$ & $\mathbf{0.26\pm0.00}$ & $\mathbf{0.95\pm0.00}$ & $0.91\pm0.00$ \\
\bottomrule
\end{tabular}
\end{table}

\begin{table}[H]
\centering
\caption{Spherical GP rotated by $60^\circ$ north, $30^\circ$ east, and an additional $20^\circ$ axial tilt. This is the hardest SO(3) generalization setting reported in the paper.}
\label{gp-rotated-60-30-20}
\begin{tabular}{lccccc}
\toprule
Name & NLL & MAE & RMSE & CVG@95\% & Runtime (min) \\
\midrule
SA-TNP (No bias) & $11.69\pm7.47$ & $0.73\pm0.01$ & $0.91\pm0.01$ & $0.58\pm0.14$ & $\mathbf{0.79\pm0.00}$ \\
BSA-TNP (RBF) & $0.08\pm0.00$ & $0.22\pm0.00$ & $0.28\pm0.00$ & $0.91\pm0.00$ & $0.91\pm0.01$ \\
BSA-TNP (Geodesic) & $\mathbf{-0.01\pm0.00}$ & $\mathbf{0.20\pm0.00}$ & $\mathbf{0.26\pm0.00}$ & $\mathbf{0.95\pm0.00}$ & $0.91\pm0.00$ \\
\bottomrule
\end{tabular}
\end{table}

% SIR
\begin{table}[H]
\centering
\caption{SIR tasks on the original training domain. The reported metric is test NLL together with wall-clock training time.}
\label{sir}
\begin{tabular}{lcc}
\toprule
Name & NLL & Runtime (min) \\
\midrule
ConvCNP & $0.24\pm0.00$ & $130.07\pm1.66$ \\
TNP-D & $\mathbf{0.19\pm0.00}$ & $57.93\pm0.15$ \\
PT-TE-TNP & $0.27\pm0.00$ & $71.45\pm0.28$ \\
BSA-TNP & $\mathbf{0.19\pm0.00}$ & $\mathbf{30.88\pm0.15}$ \\
\bottomrule
\end{tabular}
\end{table}

\begin{table}[H]
\centering
\caption{SIR tasks evaluated on domains shifted by 10 units relative to training.}
\label{sir-shifted-10}
\begin{tabular}{lcc}
\toprule
Name & NLL & Runtime (min) \\
\midrule
ConvCNP & $0.24\pm0.00$ & $2.67\pm0.01$ \\
TNP-D & $3.00\pm0.37$ & $0.95\pm0.00$ \\
PT-TE-TNP & $0.27\pm0.00$ & $0.81\pm0.01$ \\
BSA-TNP & $\mathbf{0.19\pm0.00}$ & $\mathbf{0.52\pm0.00}$ \\
\bottomrule
\end{tabular}
\end{table}

\begin{table}[H]
\centering
\caption{SIR tasks evaluated on domains scaled by a factor of 2 along each spatial axis.}
\label{sir-scaled-2x}
\begin{tabular}{lcc}
\toprule
Name & NLL & Runtime (min) \\
\midrule
ConvCNP & $0.21\pm0.00$ & $3.02\pm0.00$ \\
TNP-D & $0.24\pm0.01$ & $2.22\pm0.00$ \\
PT-TE-TNP & $0.44\pm0.02$ & $\mathbf{0.57\pm0.01}$ \\
BSA-TNP & $\mathbf{0.18\pm0.00}$ & $1.46\pm0.05$ \\
\bottomrule
\end{tabular}
\end{table}

\begin{table}[H]
\centering
\caption{Scalability on SIR as spatial resolution increases from $128\times128$ to $1024\times1024$. Runtime is reported per sample; OOM denotes out-of-memory. ConvCNP quickly runs out of memory because it must interpolate every point to every grid point. TNP-D runs out of memory because it uses $O(n^2)$ attention with post facto masking. PT-TE-TNP runs out of memory because every entry in the attention matrix must be passed through an MLP for every head for every layer.}
\label{tab:sir-scale}
\begin{tabular}{llcc}
\toprule
Resolution & Name & NLL & Runtime (s) / sample \\
\midrule
128x128 & ConvCNP & $0.05$ & $0.03$ \\
128x128 & TNP-D & $0.06$ & $0.02$ \\
128x128 & PT-TE-TNP & $0.11$ & $0.00$ \\
128x128 & BSA-TNP & $\mathbf{0.04}$ & $\mathbf{0.01}$ \\
\hline
256x256 & ConvCNP & OOM & OOM \\
256x256 & TNP-D & OOM & OOM \\
256x256 & PT-TE-TNP & $0.65$ & $\mathbf{0.04}$ \\
256x256 & BSA-TNP & $\mathbf{0.10}$ & $0.25$ \\
\hline
512x512 & ConvCNP & OOM & OOM \\
512x512 & TNP-D & OOM & OOM \\
512x512 & PT-TE-TNP & $0.66$ & $\mathbf{0.16}$ \\
512x512 & BSA-TNP & $\mathbf{0.06}$ & $3.75$ \\
\hline
1024x1024 & ConvCNP & OOM & OOM \\
1024x1024 & TNP-D & OOM & OOM \\
1024x1024 & PT-TE-TNP & OOM & OOM \\
1024x1024 & BSA-TNP & $\mathbf{0.05}$ & $\mathbf{45.31}$ \\
\bottomrule
\end{tabular}
\end{table}

% ERA5
\begin{table}[H]
\centering
\caption{ERA5 CNW: training on central Europe, validating on northern Europe, and testing on western Europe. Each task forecasts the next 6 hours from four sparse context frames separated by 6 hours.}
\label{era5-cnw}
\begin{tabular}{lccccc}
\toprule
Name & NLL & MAE & RMSE & CVG@95\% & Runtime (min) \\
\midrule
TNP-D & $0.38\pm0.06$ & $0.23\pm0.01$ & $0.29\pm0.01$ & $0.84\pm0.01$ & $\mathbf{82.01\pm0.02}$ \\
PT-TE-TNP & $0.17\pm0.02$ & $0.22\pm0.00$ & $0.28\pm0.01$ & $0.90\pm0.01$ & $85.90\pm0.02$ \\
BSA-TNP & $\mathbf{-0.07\pm0.04}$ & $\mathbf{0.18\pm0.01}$ & $\mathbf{0.23\pm0.01}$ & $\mathbf{0.91\pm0.01}$ & $82.72\pm0.05$ \\
\bottomrule
\end{tabular}
\end{table}

\begin{table}[H]
\centering
\caption{ERA5 CWN: training on central Europe, validating on western Europe, and testing on northern Europe. Each task forecasts the next 6 hours from four sparse context frames separated by 6 hours.}
\label{era5-cwn}
\begin{tabular}{lccccc}
\toprule
Name & NLL & MAE & RMSE & CVG@95\% & Runtime (min) \\
\midrule
TNP-D & $0.67\pm0.06$ & $0.25\pm0.01$ & $0.32\pm0.01$ & $0.78\pm0.01$ & $\mathbf{82.06\pm0.02}$ \\
PT-TE-TNP & $0.13\pm0.02$ & $0.21\pm0.00$ & $0.27\pm0.00$ & $0.90\pm0.01$ & $85.90\pm0.01$ \\
BSA-TNP & $\mathbf{-0.13\pm0.04}$ & $\mathbf{0.17\pm0.01}$ & $\mathbf{0.22\pm0.01}$ & $\mathbf{0.91\pm0.01}$ & $82.64\pm0.04$ \\
\bottomrule
\end{tabular}
\end{table}

% Num. Basis ablations
\begin{table}[H]
\centering
\caption{Effect of varying the number of spatial and temporal bias basis functions in BSA-TNP on 2D GP, SIR, and ERA5 CNW.}
\label{num-basis-ablation}
\begin{tabular}{lccc}
\toprule
Num. Basis & 2D GP NLL & SIR NLL & ERA5 CNW NLL \\
\midrule
1 & $-0.28\pm0.00$ & $\mathbf{0.19\pm0.00}$ & $-0.05\pm0.02$ \\
3 & $\mathbf{-0.32\pm0.01}$ & $\mathbf{0.19\pm0.00}$ & $-0.01\pm0.04$ \\
5 & $\mathbf{-0.32\pm0.00}$ & $\mathbf{0.19\pm0.00}$ & $\mathbf{-0.07\pm0.05}$ \\
10 & $\mathbf{-0.32\pm0.00}$ & $\mathbf{0.19\pm0.00}$ & $-0.04\pm0.03$ \\
\bottomrule
\end{tabular}
\end{table}

\begin{table}[H]
\centering
\caption{ERA5 SVGP. Mean and standard deviation reported across 5 BSA-TNP seeds, each evaluated on 32 single-task test batches, yielding 160 held-out western Europe tasks per method in total. For this comparison, we reuse the ERA5 benchmark from the main text: BSA-TNP is trained on central Europe, validated on northern Europe, and evaluated on western Europe over random 30-hour windows sampled in 6-hour increments. Each task contains 4 context frames and 1 forecast frame on a $7.5^\circ\times7.5^\circ$ region at $0.25^\circ$ resolution, with the number of context locations per time step sampled uniformly from integers in $[45,225)$ and all 900 pixels in the target frame used for evaluation. We benchmark 5 BSA-TNP checkpoints, one per seed, against a non-amortized sparse variational GP (SVGP) baseline that is fit from scratch on every test task using elevation, hour of day, latitude, longitude, and time as inputs. The SVGP baseline is tuned separately on 4 held-out validation tasks over the number of inducing points, optimization steps, learning rate, and initial lengthscale; the best configuration uses 64 inducing points, 500 SVI steps, and learning rate 1e-3. The reported SVGP fit time is the per-task optimization time, and the predict time is the additional time required to evaluate the fitted posterior on the 900 target pixels.}
\label{tab:era5-bsa-vs-svgp-results}
\begin{tabular}{lcccccc}
\toprule
Method & NLL & MAE & RMSE & CVG@95\% & Fit Time (s) & Predict Time (s) \\
\midrule
BSA-TNP & $\mathbf{-0.01 \pm 0.24}$ & $\mathbf{0.19 \pm 0.04}$ & $\mathbf{0.23 \pm 0.04}$ & $\mathbf{0.90 \pm 0.04}$ & -- & -- \\
SVGP & $1.31 \pm 0.38$ & $0.28 \pm 0.01$ & $0.34 \pm 0.01$ & $0.71 \pm 0.05$ & $55.45 \pm 0.68$ & $0.02 \pm 0.00$ \\
\bottomrule
\end{tabular}
\end{table}

\begin{table}[H]
\caption{Beijing Air Quality bias/embedding ablation. BSA-TNP variants compare learned location embeddings only, invariant bias only, or both on the same 48-hour-to-6-hour forecasting task.}
\label{tab:beijing-air-quality-extended}
\begin{tabular}{lccccc}
\toprule
Name & NLL & MAE & RMSE & CVG@95\% & Runtime (min) \\
\midrule
TNP-D & $-1.33\pm0.15$ & $0.04\pm0.00$ & $0.06\pm0.00$ & $0.88\pm0.01$ & $9.53\pm0.02$ \\
PT-TE-TNP & $-1.60\pm0.04$ & $0.04\pm0.00$ & $0.06\pm0.00$ & $\mathbf{0.92\pm0.00}$ & $23.39\pm0.21$ \\
BSA-TNP (embed) & $-1.75\pm0.04$ & $\mathbf{0.03\pm0.00}$ & $0.06\pm0.00$ & $0.90\pm0.00$ & $\mathbf{8.37\pm0.01}$ \\
BSA-TNP (bias) & $\mathbf{-1.82\pm0.03}$ & $\mathbf{0.03\pm0.00}$ & $\mathbf{0.05\pm0.00}$ & $0.91\pm0.00$ & $12.65\pm0.04$ \\
BSA-TNP (embed+bias) & $-1.75\pm0.04$ & $\mathbf{0.03\pm0.00}$ & $\mathbf{0.05\pm0.00}$ & $0.90\pm0.01$ & $12.81\pm0.07$ \\
\bottomrule
\end{tabular}
\end{table}

\begin{table}[H]
\caption{Gneiting GP bias/embedding ablation. BSA-TNP variants compare embeddings, invariant bias, and both on the synthetic partially stationary spatiotemporal benchmark with a nonseparable Gneiting covariance.}
\label{tab:gneiting-gp-extended}
\begin{tabular}{lccccc}
\toprule
Name & NLL & MAE & RMSE & CVG@95\% & Runtime (min) \\
\midrule
TNP-D & $1.45\pm0.01$ & $0.82\pm0.01$ & $1.03\pm0.01$ & $0.94\pm0.00$ & $\mathbf{2.25\pm0.02}$ \\
PT-TE-TNP & $1.46\pm0.00$ & $0.83\pm0.00$ & $1.03\pm0.00$ & $\mathbf{0.95\pm0.00}$ & $7.37\pm0.03$ \\
BSA-TNP (embed) & $1.46\pm0.00$ & $0.83\pm0.00$ & $1.04\pm0.00$ & $0.94\pm0.00$ & $2.65\pm0.01$ \\
BSA-TNP (bias) & $1.39\pm0.00$ & $0.78\pm0.00$ & $0.99\pm0.00$ & $0.94\pm0.00$ & $3.09\pm0.03$ \\
BSA-TNP (embed+bias) & $\mathbf{1.38\pm0.00}$ & $\mathbf{0.77\pm0.00}$ & $\mathbf{0.98\pm0.00}$ & $\mathbf{0.95\pm0.00}$ & $3.09\pm0.01$ \\
\bottomrule
\end{tabular}
\end{table}

\newpage
\section{Invariance vs. Equivariance}\label{appendix:invariance-vs-equivariance}
\paragraph{Why invariance (not equivariance) for TNPs.}
In our functional view, a TNP approximates the posterior predictive map
\(\pi:(\vd_c,\vf_c,\vd_t)\mapsto \mathcal{L}\!\left(F(\vd_t)\mid F(\vd_c)=\vf_c\right)\).
When the underlying process is stationary (translation-invariant), the correct symmetry of this map is
\begin{equation*}
  \pi(g\!\act\!\vd_c,\vf_c,\,g\!\act\!\vd_t) \;=\; \pi(\vd_c,\vf_c,\vd_t),
\end{equation*}
i.e., \emph{invariance} under jointly translating both context and target inputs. By contrast, \emph{equivariance} is the property
\(h(g\!\act\!\vd)=g\!\act\!h(\vd)\), which is appropriate when the model's output is itself a field on a fixed canvas that should transform (e.g., ConvCNPs on a latent grid or PT-TNPs with pseudo-tokens). TNPs do not output a shifted copy of a function; they return the conditional law at queried coordinates. Hence stationarity implies that predictions at translated queries, given translated contexts, are \emph{identically distributed}, not shifted, and the right inductive bias for TNPs is \textbf{translation invariance of the posterior predictive map}, which we implement via \(G\)-invariant embeddings/biases rather than output equivariance.

\section{Experimental Setup}\label{appendix:exp}
All experiments were run on a single NVIDIA GTX 4090 24GB GPU. Each benchmark was run 5 times using a different seed for each run.

\paragraph{2D GPs:}\label{appendix:gp-setup} We base the 2D GP experiments on a 64x64 grid, even though when training and testing we uniformly sample from $[-2, 2]^2$ to enable the models to learn lengthscales that fall below the grid resolution. We sample the number of context points uniformly from integers in $[128, 512)$, which corresponds to approximately 3\%-12.5\% of pixels on a 64x64 grid. We test on a separate 1024 points, corresponding to 25\% of the points on a 64x64 grid. We use batch size 8 and observation noise of 0.1 for context points. We use a squared exponential kernel with lengthscale sampled from $\text{Beta}(3.0, 7.0)$, which has a mean and median of $\approx$ 0.3. This is a more challenging benchmark than most NPs use since 50\% of the lengthscales fall below 0.3 and only 10\% fall above 0.5. In practice we found that most NPs could easily learn lengthscales above 0.4-0.5. We do not sample the variance because the data can always be standardized using the variance. We train over 100K batches and validate every 10K steps on 5K batches. We use the AdamW optimizer with $\beta_1=0.9$, $\beta_2=0.999$, and weight decay 1e-4. We clip maximum gradient norms at 0.5 and use a cosine learning rate schedule which starts at 1e-4 and decays to 2e-5. The shifting and scaling 2D GP benchmarks use the models trained under this regime and are simply evaluated on new domains.

\paragraph{Spherical GPs with Rotational Invariance:}\label{appendix:rot-setup}
The spherical experiments are based on a 64x64 grid of (longitude, latitude) pairs. Three BSA-TNP variants are compared: a translation-invariant variant with RBF network-based bias, an SO(3)-invariant variant with the squared exponential geodesic bias given by $\sum_m a_m\exp(-b_m |d_{geo}(s,s')|^2)$ where $d_{geo}$ is the great-circle distance, and an unbiased variant that only embeds locations. Like the 2D GP setup, we sample the number of context points uniformly from integers in $[128, 512)$, representing approximately 3\% and 12.5\% of points on a 64x64 grid, and test on 1024 separate test points, representing 25\% of points. The locations are sampled uniformly from $[-10^\circ,10^\circ]^2$. For evaluation, these points are either (1) left unrotated, (2) rotated by $60^\circ$ north and $30^\circ$ east, or (3) rotated by $60^\circ$ north, $30^\circ$ east, and $20^\circ$ along the axis given by $(0^\circ,0^\circ), (180^\circ,0^\circ)$ in the original coordinate system. In terms of Euler angles, this corresponds to intrinsic "yxz" rotations by $(-60^\circ, 30^\circ, 0^\circ)$ and $(-60^\circ, 30^\circ, 20^\circ)$, respectively. The GP kernel is exponential with great-circle distance. The variance is fixed at 1.0 and the lengthscale is sampled from $\mathrm{InverseGamma}(3, 30)$. The larger scale relative to the 2D GP experiments is used to emphasize the difference between translations and spherical rotations --- as the scale decreases the curvature of the considered region becomes negligible and there is no significant difference between translations with scaling and SO(3) rotations. While RBF network-based biases perform well, geodesic network-based biases outperform as the scale increases and the rotation becomes more pronounced. \autoref{tab:rot-main} shows the results of this benchmark and \autoref{fig:rot-main} provides a visualization.

\paragraph{Susceptible-Infected-Recovered (SIR):}\label{appendix:sir-setup} For the SIR benchmark, we simulate tasks from a SIR model. It is governed by an infection rate, $\beta$, a recovery rate, $\gamma$, and the number of initial infections, $\omega$. We sample $\beta\sim\operatorname{Beta}(2, 8)$, $\gamma\sim\operatorname{InvGamma}(5, 0.4)$, and $\omega\sim\operatorname{randint}(1,5)$. The infection rate, $\beta$, is decreased as an inverse function of distance from the infected individual. In expectation, this parameter setting corresponds to an infection rate of 20\% upon exposure and a 10-day recovery period (similar to COVID-19). We roll out simulations for 25 steps and randomly sample steps to create tasks. We use a grid size of 64x64 on the domain $[-2, 2]^2$. Similar to the 2D GP benchmark, we sample the number of context points uniformly from integers in $[128, 512)$, which corresponds to approximately 3\%-12.5\% of pixels. We test on a separate 1024 points, corresponding to 25\% of the points on a 64x64 grid. We train over 100K batches, each with batch size 8, and validate every 10K steps on 5K batches. We test on 5K unseen batches. We use the AdamW optimizer with $\beta_1=0.9$, $\beta_2=0.999$, and weight decay 1e-4. We clip maximum gradient norms at 0.5 and use a cosine learning rate schedule which starts at 1e-4 and decays to 1e-5.

\paragraph{ERA5}\label{appendix:era5-setup} We largely follow the setup in \cite{pttnp} with ERA5 \cite{era5} surface air temperatures. To test learning over time and generalization across space, we train models on tasks from central Europe $[42^\circ\mathrm{N},\,53^\circ\mathrm{N}]\times[8^\circ\mathrm{E},\,28^\circ\mathrm{E}]$ and test them on tasks from northern Europe $[53^\circ\mathrm{N},\,62^\circ\mathrm{N}]\times[8^\circ\mathrm{E},\,28^\circ\mathrm{E}]$ and western Europe $[42^\circ\mathrm{N},\,53^\circ\mathrm{N}]\times[4^\circ\mathrm{W},\,8^\circ\mathrm{E}]$ over random 30-hour segments of time. When testing on northern Europe, we use western Europe as a validation set and select the best model for testing and vice versa when testing on western Europe. Unlike \cite{pttnp}, which inpaints several frames over time, we change the benchmark to forecast the weather in the next 6 hours. We also add ``hour of day'' to the feature set and increase the resolution from $0.5^\circ$ to $0.25^\circ$. All input data are standardized based on the training set, i.e. central Europe. The task consists of 4 context frames and 1 test frame, each separated by 6 hours. The inputs we use are \textbf{x}=(elevation, hour of day), \textbf{s}=(latitude, longitude), \textbf{t}=time, \textbf{f}=surface temperature. Each task consists of a 30x30 image or $7.5^\circ\times7.5^\circ$ at a resolution of $0.25^\circ$. We sample context points uniformly from integers in $[45, 225)$ from each time step in the context set, corresponding to approximately 5\% and 25\% of the number of pixels. The test set for each task consists of all 900 pixels from the target time step. We train on 100K batches, each of size 8, and validate every 5K steps on 5K batches. We test on 5K batches from the test region. Following \cite{pttnp}, we use the AdamW optimizer with $\beta_1=0.9$, $\beta_2=0.999$, and weight decay 1e-4. We clip maximum gradient norms at 0.5 and use a constant learning rate of 5.0e-4. The main script for this benchmark is \mintinline{bash}{dl4bi/benchmarks/meta_learning/era5.py}. It downloads the required ERA5 data automatically, but requires a \href{https://cds.climate.copernicus.eu/how-to-api}{CDS API key}.

\paragraph{Beijing Multi-City Air Quality:}\label{appendix:beijing-setup}
For this benchmark, we use the Beijing Multi-Site Air Quality dataset from the UCI repository \citep{beijing}. After merging in station coordinates, we drop rows with missing values, sort observations in temporal order, and use an 80\%/10\%/10\% chronological train-validation-test split. The fixed covariates are \textbf{x}=(year, month, day, hour, day of week, is weekend, temperature, pressure, dewpoint, rain, wind speed, wind direction), where wind direction is one-hot encoded. The spatial inputs are \textbf{s}=(latitude, longitude), the temporal input \textbf{t} is elapsed time in seconds since the start of the dataset, and the targets are \textbf{f}=(PM2.5, PM10). All numeric variables in \textbf{x}, \textbf{s}, \textbf{t}, and \textbf{f} are min-max scaled using transformers fit on the training split. Because this is a comparatively easy dataset, we train over only 10K batches, each of size 32. Most models overfit quickly, so we validate on 500 batches every 500 steps and use the model with the lowest validation score for testing. The number of context points is fixed at 624, corresponding to an intended 48-hour history for each of the 12 stations, padded by an additional 4 hours per station to account for gaps after rows with missing values are removed. The number of test points is fixed at 84, corresponding to an intended 6-hour forecast horizon plus 1 extra hour per station for the same reason. We test on 5K batches from the test split. We use the AdamW optimizer with $\beta_1=0.9$, $\beta_2=0.999$, and weight decay 1e-4. We clip maximum gradient norms at 0.5 and use a cosine learning rate schedule which starts at 1e-3 and decays to 1e-5. The main script for this benchmark is \mintinline{bash}{dl4bi/benchmarks/meta_learning/beijing_air_quality.py}. The \href{https://archive.ics.uci.edu/dataset/501/beijing+multi+site+air+quality+data}{Beijing Multi-City Air Quality} dataset must be downloaded, extracted, and the files in the \mintinline{bash}{PRSA_*} folder placed in \mintinline{bash}{dl4bi/benchmarks/meta_learning/cache/beijing_air_quality/}.

\paragraph{Gneiting Gaussian Process:}\label{appendix:gneiting-gp}
For this benchmark, we simulate tasks from a spatiotemporal GP on a $16\times16$ grid over the spatial domain $[-1,1]^2$ and 12 evenly spaced time points on $[0,1]$. Each task uses 5 consecutive time points, with the first 4 treated as context and the 5th as the forecast target. To avoid overfitting to a fixed lattice, the observations are sampled on an irregular spatial layout by uniformly drawing 256 locations in $[-1,1]^2$ and then randomly subsampling a different set of 8 to 15 context locations at each context time step. The test set consists of all 256 locations at the target time step. The covariance is given by the nonseparable Gneiting kernel
\[
\mathcal{K}\big((s,t),(s',t')\big)=\frac{\sigma^2}{g(|t-t'|)^\nu}\exp\!\left(-\left(\frac{\| (s-s')/\ell_s\|^2}{g(|t-t'|)^b}\right)^\gamma\right),\qquad g(u)=1+a u^{2\alpha},
\]
where the positive scale parameters are sampled log-uniformly and the bounded shape parameters uniformly: $\sigma^2\in[0.85,1.75]$, $\ell_s=(\ell_x,\ell_y)$ with each axis in $[0.05,0.15]$, $a\in[0.25,2.5]$, $\alpha\in[0.35,1.0]$, $b\in[0,0.75]$, $\nu\in[1.1,4.0]$, and $\gamma\in[0.55,1.0]$. These ranges keep the spatial correlation length well below the width-2 domain so that the field varies meaningfully across the 256 irregularly sampled locations, while still remaining locally smooth enough to be forecast from only 8 to 15 context points per time step. The temporal parameters $a$ and $\alpha$ span weak to strong temporal decorrelation, while $b$ and $\nu$ control how much time separation inflates effective spatial distance and attenuates marginal covariance. Restricting $\gamma$ to $[0.55,1.0]$ interpolates between rougher exponential-like and smoother squared-exponential-like decay without producing pathological samples. Observation noise is sampled log-uniformly from $[0.01,0.06]$, which perturbs the targets without overwhelming the latent field.

To make the process only partially stationary, we add a deterministic mean function of the form $6.0 + 0.9\,m(s,t)$, where $m$ is a weighted sum of shared basis functions. The per-task weights are sampled as a bias term in $[-0.15,0.15]$, a terrain coefficient in $[-0.75,0.75]$, sine and cosine clock weights in $[-0.55,0.55]^2$, a second-harmonic coefficient in $[-0.25,0.25]$, a terrain-time interaction weight in $[-0.35,0.35]$, a linear trend in $[-0.10,0.10]$, and a moving-hotspot amplitude in $[-0.55,0.55]$. The shared terrain is fixed across tasks via three Gaussian bumps plus a sinusoidal ripple, while the hotspot begins at $(-0.55,-0.25)$, moves with velocity $(1.05,0.55)$, and has width $(0.24,0.18)$. Keeping these deterministic coefficients on roughly the same scale as the GP standard deviation creates a benchmark that is only partially stationary: models must infer local covariance structure, but can still benefit from learning repeatable large-scale spatiotemporal patterns. We train for 50K batches with batch size 4, validate every 5K steps on 500 batches, and test on 500 unseen batches. We use AdamW with $\beta_1=0.9$, $\beta_2=0.999$, weight decay 1e-4, gradient clipping at 0.5, and a cosine learning rate schedule from 1e-3 to 1e-5.

\section{Complexity and Estimated FLOPS}\label{appendix:scaling}
In \autoref{tab:complexity}, we provide complexity estimates for the primary models compared in this paper. In \autoref{tab:flops}, we provide estimated training and inference GFLOPs for each model. These are estimated using \href{https://docs.jax.dev/en/latest/aot.html}{JAX's cost analysis} on functions lowered to HLO and compiled. These figures change from benchmark to benchmark based on the batch size and distribution of test and context points. We provide estimates for 2D GPs since this was our most generic benchmark. Notably, while ConvCNP has the lowest estimated GFLOPs for training and inference, this does not translate into wall-clock speed, likely due to poorer high-bandwidth memory (HBM) accesses and the large induced grid that must be encoded and decoded for every task.

\begin{table}[h!]
\caption{Time and space complexity. $n_c$ is the number of context points, $n_t$ is number of test points, $n_i$ is number of inducing points, and $n_b$ is block size.}
\label{tab:complexity}
\begin{center}
\begin{tabular}{lcc}
\toprule
\textbf{Model} & \textbf{Time} & \textbf{Space} \\
\midrule
ConvCNP & $\mathcal{O}(n_cn_i + n_in_t)$ & $\mathcal{O}(n_cn_i + n_in_t)$ \\
TNP-D & $\mathcal{O}((n_c+n_t)^2)$ & $\mathcal{O}((n_c + n_t)^2)$ \\
PT-TE-TNP & $\mathcal{O}(k(n_c+n_t))$ & $\mathcal{O}(k(n_c + n_t))$ \\
BSA-TNP & $\mathcal{O}(n_c^2 + n_cn_t)$ & $\mathcal{O}(n_b^2)$ \\
\bottomrule
\end{tabular}
\end{center}
\end{table}

\begin{table}[h!]
\caption{GFLOPs estimates on 2D GP and wall-clock times. GFLOPs are estimated using JAX's cost analysis on functions lowered to HLO and then compiled for a device.}
\label{tab:flops}
\begin{center}
\begin{tabular}{lcccc}
\toprule
Model & Train GFLOPS & Infer GFLOPS & Train (min) & Infer (batch/s) \\
\midrule
ConvCNP & $\mathbf{88}$ & $\mathbf{23}$ & $121.10\pm1.02$ & $27.56\pm0.04$  \\
TNP-D & $217$ & $72$ & $65.27\pm0.04$ & $63.95\pm0.40$ \\
PT-TE-TNP & $310$ & $101$ & $72.40\pm0.08$ & $73.35\pm0.77$ \\
BSA-TNP & $94$ & $31$ & $\mathbf{36.98\pm0.09}$ & $\mathbf{103.74\pm0.43}$ \\
\bottomrule
\end{tabular}
\end{center}
\end{table}

\section{Convergence}\label{appendix:convergence}
To visualize how bias can accelerate convergence, we provide the validation NLL versus iteration in \autoref{fig:sir-convergence} from the SIR benchmark, which demonstrates that \textit{G}-invariant bias permits almost immediate convergence to the optimal solution. Furthermore, this performance was maintained under shifting and scaling, unlike TNP-D.

\begin{figure}[h!]  % [t], [b], or [h] for placement
  \centering
  \includegraphics[width=0.50\linewidth]{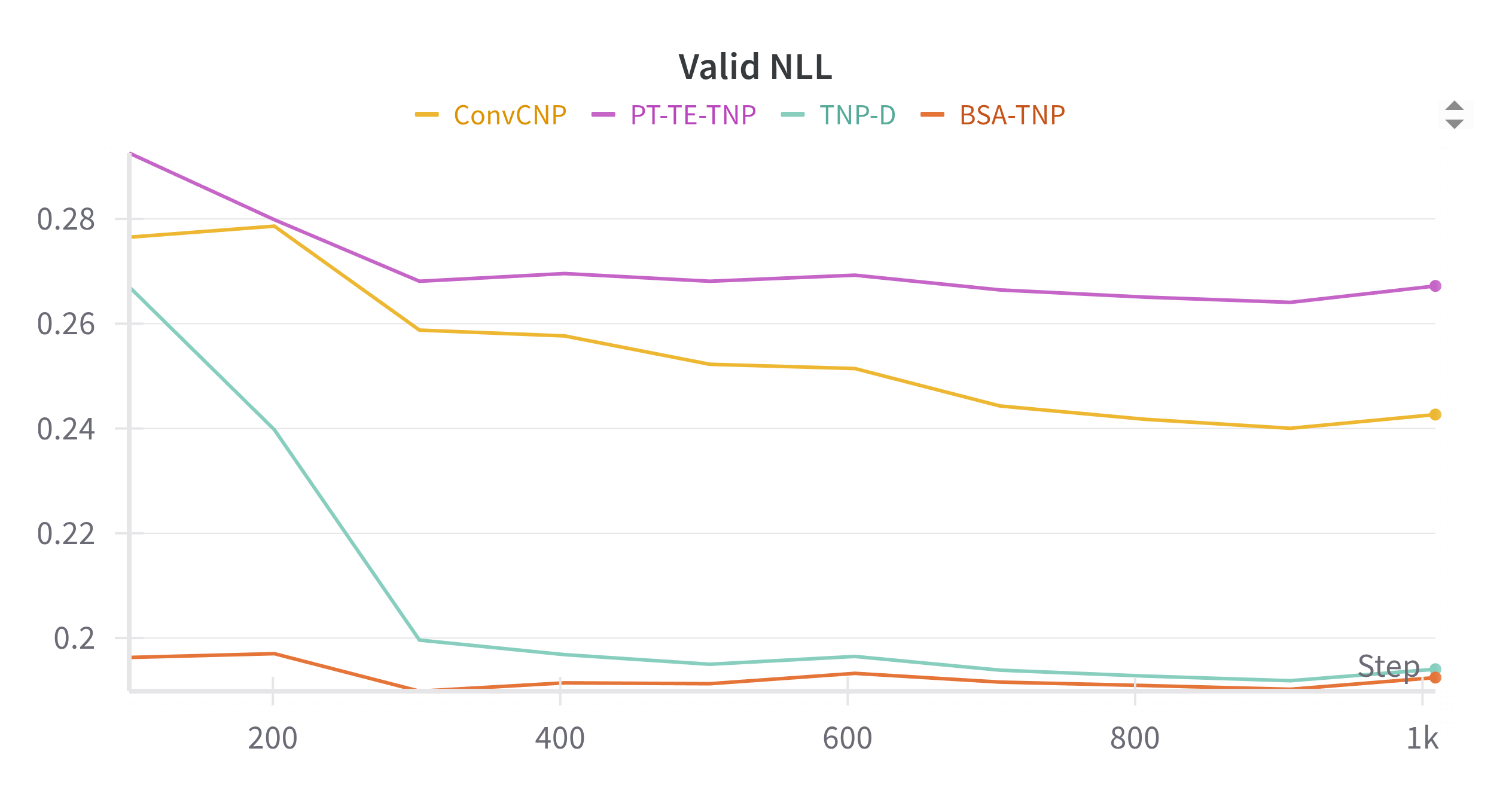}
  \caption{Validation NLL on the SIR benchmark for seed 91. All seeds exhibited the same pattern.}
  \label{fig:sir-convergence}
\end{figure}

\section{Biased Flex Attention (BFA) vs. Biased Scan Attention (BSA)}\label{appendix:flex-vs-scan}
In order to evaluate the efficiency of Biased Scan Attention (BSA) in JAX vs. Biased Flex Attention (BFA) in PyTorch, we vary the sequence length $L$ for keys and queries while using a batch size of 32, 4 attention heads, 64-dimensional embeddings, and 2-dimensional space, i.e. $(B=32, H=4, D=64, D_s=2)$, corresponding to the default case for most benchmarks.

Each experiment was run 100 times for each sequence length and the means for the forward and backward passes are plotted below. We see that while BFA is slower than BSA for shorter sequences, it becomes about 5 times faster on the forward pass, i.e. during inference. However, for the backward pass, BFA is 14-170 times slower than BSA. \textbf{This means that to train on 100,000 batches at $L=4096$ (in a single-layer transformer), BFA would take approximately 7 days while BSA would take 0.5 days, and both would run inference on a test set in less than a tenth of a second.}

\begin{figure}[h!]
  \centering
  \begin{subfigure}[b]{0.48\linewidth}
    \centering
    \includegraphics[width=\linewidth]{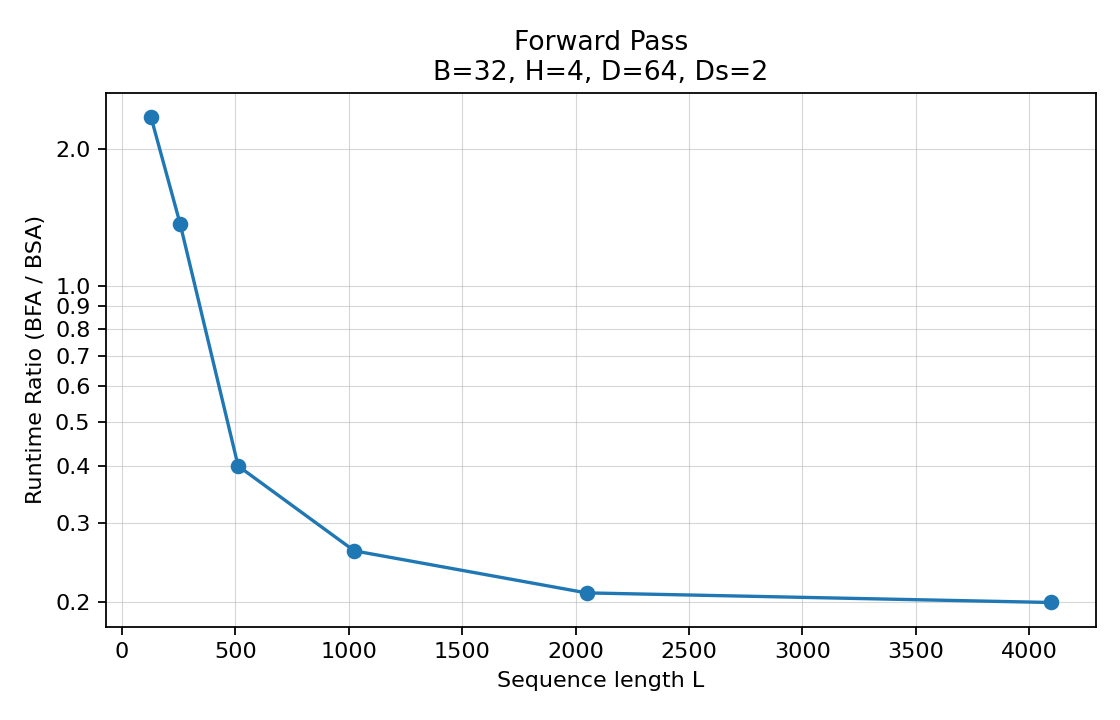}
    \caption{Forward Pass Runtime Ratio (BFA / BSA)}
    \label{fig:bfa-vs-bsa-forward}
  \end{subfigure}
  \hfill
  \begin{subfigure}[b]{0.48\linewidth}
    \centering
    \includegraphics[width=\linewidth]{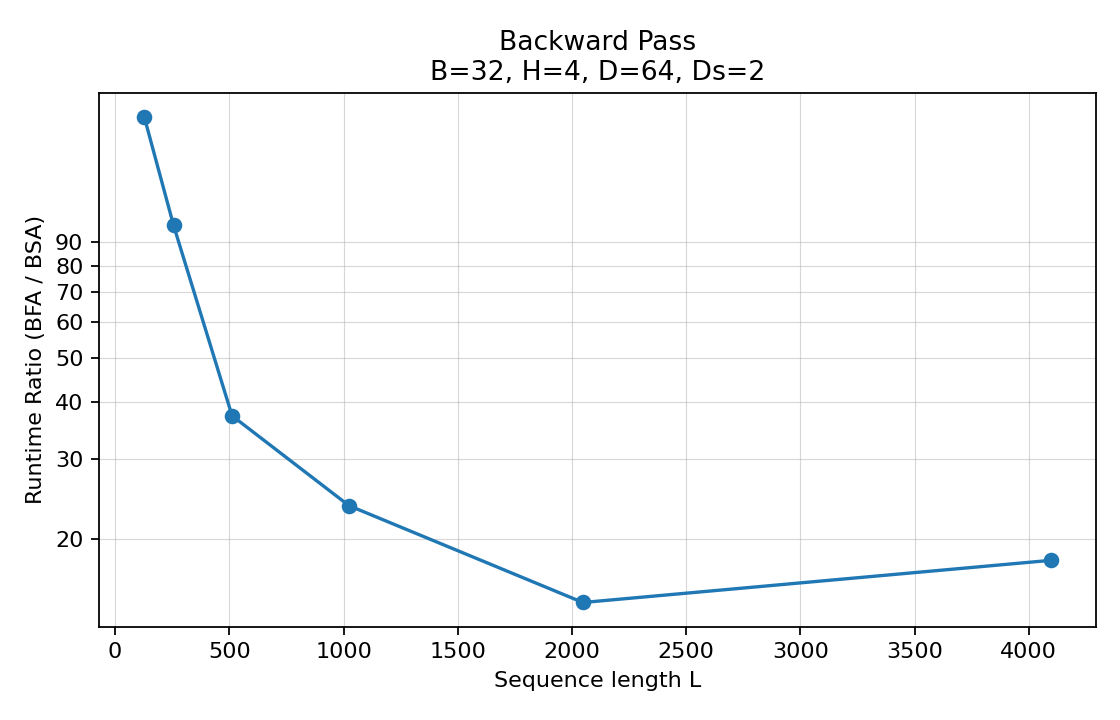}
    \caption{Backward Pass Runtime Ratio (BFA / BSA)}
    \label{fig:bfa-vs-bsa-backward}
  \end{subfigure}
  \caption{Comparison of runtime ratios for forward and backward passes.}
  \label{fig:bfa-vs-bsa}
\end{figure}
\clearpage

\newpage
\section{Multiresolution 2D GPs}\label{appendix:multires-gp}
The following figure visualizes a batch of multiresolution 2D GP tasks with predictions from BSA-TNP. By using a coarser resolution for the area surrounding the target region, the number of pixels is reduced by $\approx$ 78\%.
\begin{figure}[h!]  % [t], [b], or [h] for placement
  \centering
  \includegraphics[scale=0.125]{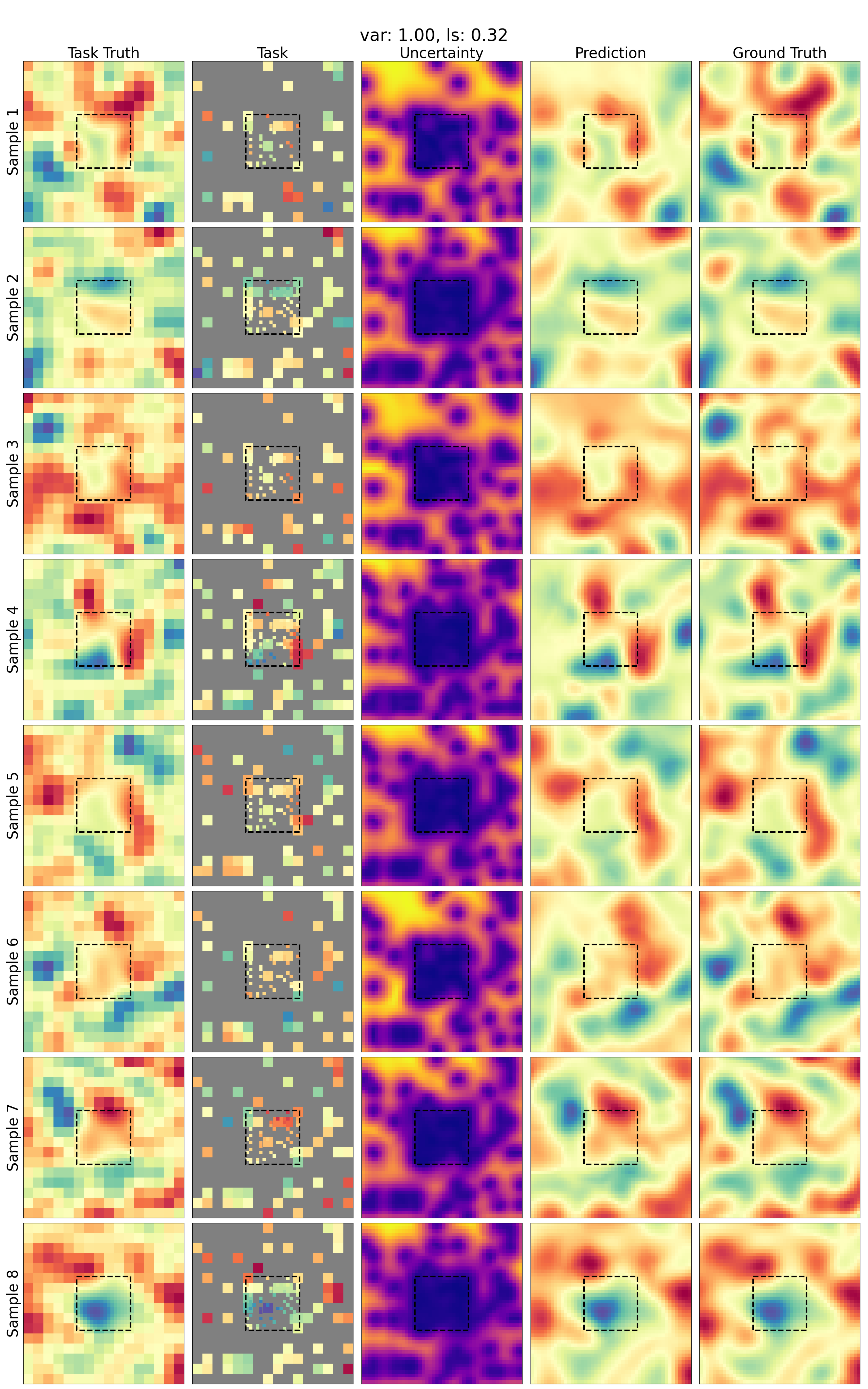}
  \caption{An example of a batch of multiresolution 2D GP tasks. Predictions are from BSA-TNP.}
  \label{fig:mulitres-gp}
\end{figure}

\newpage
\section{Group Invariance Theory Proof}\label{appendix:G-inv-iff-proof}

In this section we prove Theorem \ref{theorem:iff}, using the same notation and definitions introduced in Section \ref{sec:G-invariance}.

\begin{proof}
  We assume a probability space $(\Omega, \calF, \P)$ with the stochastic process $F$ taking values in $Y = \reals^n$ endowed with the Borel $\sigma$-algebra, so that regular conditional distributions exist.
  Let $\mu_\vd$ denote the law of $F(\vd)$, and $\mu_{\vd_t \given \vd_c,
  \vf_c}$ the conditional law of $F(\vd_t) \given F(\vd_c) = \vf_c$,
  that is the probability measures defined by $\mu_\vd(A) = \P(F(\vd) \in A)$
  and $\mu_{\vd_t \given \vd_c, \vf_c}(A) = \P(F(\vd_t) \in A \given F(\vd_c) = \vf_c)$. Details on the definitions of $\sigma$-algebras, measures, Lebesgue-Stieltjes integration, and conditional probabilities can be found e.g.~in \cite{klenke_probability_2020}.
  
  $\implies$: Assume $F$ is $G$-stationary, then for all measurable $A \subseteq Y^{|\vd_t|}$ and $B \subseteq Y^{|\vd_c|}$,
  by (1) consistency of regular conditional distributions, and (2) $G$-stationarity of $F$ so that $\mu_\vd = \mu_{g \act \vd}$, it follows that:
  \begin{align*}
  \int_{\vf_c\in B} \mu_{\vd_t \given \vd_c, \vf_c}(A) \dmu_{\vd_c}(\vf_c) 
  &\stackrel{(1)}= \mu_{(\vd_t, \vd_c)}(A \times B) \\
    &\stackrel{(2)}= \mu_{g\act (\vd_t, \vd_c)}(A \times B) \\
    &\stackrel{(1)}= \int_{\vf_c\in B} \mu_{g \act \vd_t \given g \act \vd_c, \vf_c}(A) \dmu_{g \act \vd_c}(\vf_c) \\
    &\stackrel{(2)}= \int_{\vf_c\in B} \mu_{g \act \vd_t \given g \act \vd_c, \vf_c}(A) \dmu_{\vd_c}(\vf_c).
  \end{align*}
  This equality holds for any measurable $B$. By Lemma \ref{lemma:cond_eq} we get that for all $A$, for $\mu_{\vd_c}$-almost-every $\vf_c$, 
  $\mu_{g \act \vd_t \given g \act \vd_c, \vf_c}(A) = \mu_{\vd_t \given \vd_c, \vf_c}(A)$.
  Lemma \ref{lemma:swap} allows us to swap the quantifiers and get that for $\mu_{\vd_c}$-almost-every $\vf_c$ the laws 
  $\mu_{g \act \vd_t \given g \act \vd_c, \vf_c}$ and $\mu_{\vd_t \given\vd_c, \vf_c}$ are equal. 
  Conditional distributions are defined with respect to equality almost everywhere in the conditioning argument, therefore we may conclude that
  the distributions of $F(\vd_t) \given F(\vd_c) = \vf_c$ and $F(g \act \vd_t) \given F(g \act \vd_c) = \vf_c$ are equal, and thus, the posterior predictive map $\pi$ is $G$-invariant.

 $\impliedby$: Assume the posterior predictive map $\pi$ is $G$-invariant. Let $\vd_c$ and $\vf_c$ be the empty set. We immediately get that \[F(g\act\vd_t) =\pi(g \act \vd_c, \vf_c, g \act \vd_t) = \pi(\vd_c, \vf_c, \vd_t)=F(\vd_t)\]
 and hence, $F$ is $G$-invariant.
 
\end{proof}
\begin{lemma}\label{lemma:cond_eq}
  If $\int_E f \dmu = \int_E g \dmu$ for all measurable sets $E$ and a non-negative measure $\mu$, then $f = g$ $\mu$-almost-everywhere.
\end{lemma}
\begin{proof}
  Consider the set $E_n = \{ x : f(x) - g(x) \ge \frac1n \}$, which is measurable since $f, g$ are integrable. Then $0 = \int_{E_n} (f - g) \dmu \ge \frac1n \mu(E_n) \ge 0$, so it must be that $\mu(E_n) = 0$. As a countable union of null sets, it must be that $\{x : f(x) - g(x) > 0\} = \bigcup_{n=1}^\infty E_n$ is null. By a symmetric argument, it must be that $\{x : g(x) - f(x) > 0 \}$ is null. Therefore $f(x) = g(x)$ for $\mu$-almost-every $x$.
\end{proof}
\begin{lemma}\label{lemma:swap}
    Let $\mu$ be a probability measure over $\calB(\reals^d)$ and 
    $f, g: \calB(\reals^k) \times \reals^d \to [0,1]$ Markov kernels, that is functions, such that
    \begin{enumerate}
        \item for $\mu$-almost-every $x$, $f(\cdot, x), g(\cdot, x)$ are probability measures,
        \item for all $A$, $f(A, \cdot), g(A, \cdot)$ are measurable functions.
    \end{enumerate}

    If for all $A$, for $\mu$-almost-every $x$, $f(A, x) = g(A, x)$,
    then $f(\cdot, x) = g(\cdot, x)$ for $\mu$-almost-every $x$.
\end{lemma}
\begin{proof}
    First consider sets of the form $A = (a_1, b_1) \times \ldots \times (a_k, b_k)$ with $a_i, b_i \in \mathbb{Q}$ for all $i = 1, \ldots, k$, that is, open rectangles in $\reals^k$ with rational coordinates,
    and enumerate this countable family as $A_n, n \in \naturals$. By assumption, $E_n = \{ x : f(A_n ,x) \neq g(A_n, x) \}$ and $F= \{x : f(\cdot, x), g(\cdot, x) \text{ are not probability measures} \}$ are null with respect to $\mu$,
    so $N := \bigcup_{n \in \naturals} E_n \cup F$ is null. 

    Now for fixed $x \in \reals^d \setminus N$, consider the set $\mathcal{E} = \{B \in \calB(\reals^k) : f(B, x) = g(B, x) \}$. Since $f(\cdot, x), g(\cdot, x)$ are measures, $\mathcal{E}$ is a $\sigma$-algebra.
    Moreover, by definition of $N$, $\mathcal{E}$ contains the set of all rational open rectangles in $\reals^k$. Since these rectangles generate $\calB(\reals^k)$, it must be that $\calB(\reals^k) \subseteq \mathcal{E}$ and hence the measures $f(\cdot, x)$ and $g(\cdot, x)$ are equal. Since this holds for all $x \in \reals^k \setminus N$, the proof is complete.
\end{proof}
\section{BSA-TNP G-Invariance Proof}\label{appendix:BSA-inv-proof}
In this section we prove Theorem \ref{theorem:BSA-inv}, using the same notation, definitions, and BSA-TNP architecture module names introduced in Sections \ref{sec:G-invariance}, \ref{sec:bsa-tnp}.

\begin{proof}
    This follows by tracing $(g \act \vd_c, g \act \vd_t)$ through the BSA-TNP architecture.
    \begin{enumerate}
        \item By the assumption of the theorem $\textbf{Embed}_\varphi$ is $G$-invariant in $\vd_c, \vd_t$. Here, $\varphi$ represents the collection of all embedding parameters.
        \item The kernels used in the attention bias are $G$-invariant. As no other calculation within the attention mechanism involves $\vd_{\{c,t\}}$, we get that $\mathrm{BSA}$ is $G$-invariant. Thus, $\mathrm{BSA}(\qs, \vd_q, \ks,\vd_k) = \mathrm{BSA}(\qs, g \act \vd_q, \ks, g\act \vd_k)$. Consequently, 
        $\ve_c' = \mathrm{BSA}(\ve_c, g \act \vd_c, \ve_c, g\act \vd_c) = \mathrm{BSA}(\ve_c,  \vd_c, \ve_c,  \vd_c)$ and $\ve_t' = \mathrm{BSA}(\ve_t, g \act \vd_t, \ve_c, g\act \vd_c) = \mathrm{BSA}(\ve_t,  \vd_t, \ve_c,  \vd_c)$. As a result, the KRBlock's output $\ve'_{\{c,t\}}$ is $G$-invariant with respect to $\vd_c, \vd_t$.
        \item The projection head only takes the encoding $\ve'_t$ output by the final KRBlock, and is thus agnostic to $\vd_c, \vd_t$.
    \end{enumerate}
    By the above, BSA-TNP consists only of $G$-invariant operations: $\textbf{Embed}_\varphi$, followed by an arbitrary number of KRBlocks, and the projection head, and therefore stacking these operations results in a $G$-invariant model in $\vd_c, \vd_t$.
\end{proof}

\end{document}